\documentclass[journal]{IEEEtran}

\usepackage{epsfig,graphics,subfigure}
\usepackage{amsmath,amssymb,amsthm,multirow,balance,dsfont,hyperref}
\usepackage{cite,cleveref}
\usepackage[table]{xcolor}
\usepackage{colortbl}
\usepackage{array}
\usepackage[ruled,vlined]{algorithm2e}
\usepackage{mathrsfs}
\usepackage{empheq}
\usepackage[mathscr]{euscript}
\usepackage{rotating}
\usepackage{caption}
\usepackage{hyperref}
\usepackage{cleveref}
\renewcommand{\thetable}{\arabic{table}}
\crefname{table}{Table}{Tables}
\DeclareFontFamily{U}{mathx}{\hyphenchar\font45}
\DeclareFontShape{U}{mathx}{m}{n}{
      <5> <6> <7> <8> <9> <10>
      <10.95> <12> <14.4> <17.28> <20.74> <24.88>
      mathx10
      }{}
\DeclareSymbolFont{mathx}{U}{mathx}{m}{n}
\DeclareFontSubstitution{U}{mathx}{m}{n}
\DeclareMathAccent{\widecheck}{0}{mathx}{"71}

\newtheorem{theorem}{Theorem}

\newtheorem{example}{Example}

\newtheorem{assumption}{Assumption}

\usepackage{color}

\newcommand{\boldh}{\boldsymbol{h}}

\newcommand{\bd}{\boldsymbol{d}}
\newcommand{\bg}{\boldsymbol{g}}
\newcommand{\bs}{\boldsymbol{s}}

\newcommand{\bu}{\boldsymbol{u}}
\newcommand{\bv}{\boldsymbol{v}}
\newcommand{\bw}{\boldsymbol{w}}

\newcommand{\bx}{\boldsymbol{x}}

\newcommand{\bgamma}{\boldsymbol{\gamma}}
\newcommand{\bpsi}{\boldsymbol{\psi}}
\newcommand{\bphi}{\boldsymbol{\phi}}

\newcommand{\bH}{\boldsymbol{H}}

\newcommand{\bR}{\boldsymbol{R}}

\newcommand{\cN}{\mathcal{N}}

\newcommand{\cw}{{\scriptstyle\mathcal{W}}}

\newcommand{\ccw}{{\scriptscriptstyle\mathcal{W}}}

\newcommand{\bcF}{\boldsymbol{\cal{F}}}
\newcommand{\bcw}{\boldsymbol{\cw}}

\newcommand{\bwt}{\widetilde \bw}

\newcommand{\expec}{\mathbb{E}}

\newcommand{\col}{\text{col}}

\newcommand{\diag}{\text{diag}}
\newcommand{\sign}{\text{sign}}
\newcommand{\prox}{\text{prox}}

\usepackage{etoolbox}
\AtBeginEnvironment{tabular}{\scriptsize}
\newcolumntype{C}[1]{>{\centering\arraybackslash}m{#1}}

\begin{document}
\title{A non-smooth regularization framework \\for learning over multitask graphs}
\author{Yara Zgheib, \IEEEmembership{Student Member, IEEE}, Luca Calatroni, \IEEEmembership{Member, IEEE}\\
 Marc Antonini, \IEEEmembership{Member, IEEE}, Roula Nassif, \IEEEmembership{Member, IEEE} \thanks{

The work of Y. Zgheib and R. Nassif was supported in part by ANR JCJC grant ANR-22-CE23-0015-01 (CEDRO project).

Y. Zgheib, M. Antonini, and R. Nassif are with Universit\'e C\^ote d'Azur, I3S Laboratory, CNRS,  France (email: $\{$yara.zgheib, marc.antonini, roula.nassif$\}$@unice.fr). L. Calatroni is with Machine Learning Genoa Center (MaLGa), Department of Computer Science, University of Genoa, Italy (email: luca.calatroni@unige.it).}}

\maketitle

\begin{abstract}
In this work, we consider learning over multitask graphs, where each agent aims to estimate its own parameter vector, referred to as its \emph{task} or \emph{objective}. Although agents seek distinct objectives, collaboration among them can be beneficial in scenarios where relationships between tasks exist. Among the various approaches to promoting relationships between tasks and, consequently, enhancing collaboration between agents, one notable method is regularization. While previous multitask learning studies have focused on smooth regularization to enforce graph smoothness, this work explores non-smooth regularization techniques that promote sparsity, making them particularly effective in encouraging piecewise constant transitions on the graph. We begin by formulating a global regularized optimization problem, which involves minimizing the aggregate sum of individual costs, regularized by a general non-smooth term designed to promote piecewise-constant relationships between the tasks of neighboring agents. Based on the forward-backward splitting strategy, we propose a decentralized learning approach that enables efficient solutions to the regularized optimization problem. The proposed approach allows for the use of stochastic gradient vectors in place of true gradient vectors, which are typically unavailable in streaming data scenarios. Then, under convexity assumptions on the cost functions and co-regularization, we establish that the proposed approach converges in the mean-square-error sense within $O(\mu)$ (where $\mu$ is the algorithm's step-size) of the optimal solution of the globally regularized cost. For broader applicability and improved computational efficiency, we also derive closed-form expressions for commonly used non-smooth (and, possibly, non-convex) regularizers, such as the weighted sum of the $\ell_0$-norm, $\ell_1$-norm, and elastic net regularization. Finally, we illustrate both the theoretical findings and the effectiveness of the approach through simulations. 
\end{abstract}

\begin{IEEEkeywords}
Distributed optimization, multitask graph, personalized learning, non-smooth regularization, proximal operator, gradient noise, mean-square-error stability.
\end{IEEEkeywords}
\vspace{-0.4 cm}
\section{Introduction}
The rapid proliferation of smart devices, including mobile phones, wearable technology, and autonomous vehicles, has led to a significant increase in the amount of private data generated continuously from distributed sources. Due to its sensitive nature and sheer volume, this data presents both opportunities and challenges. On one hand, the analysis of this data can drive innovation in healthcare industries, smart cities, and so on. On the other hand, it raises significant concerns, such as privacy issues related to the  collection and analysis of personal data, which could lead to misuse or unauthorized access. Conventional learning approaches often struggle with these issues, as they may not fully address privacy concerns. This requires new strategies that prioritize data protection while also leveraging the power of data for innovation.

In view of these challenges, \emph{federated} and \emph{decentralized} learning  have gained in popularity in recent years. In federated learning, multiple devices (also referred to as \emph{clients},  \emph{nodes}, or \emph{agents}) collaboratively learn a shared model under the orchestration of a central server, which coordinates the process and aggregates the model updates from each client~\cite{konecny2016federated,li2020federated,kairouz2021advances}. By keeping the training data on local devices and exchanging only model updates, federated learning reduces the risk of sharing sensitive data. Decentralized learning is an alternative approach that enables collaborative learning without relying on a central server~\cite{Sayed2014adaptation,sayed2013diffusion,Nassif2020multitask,Nedic2009distributed,koloskova2019decentralized}. The prevailing decentralized setting typically assumes a communication network topology where each agent can directly communicate with its neighboring agents. In this case, model updates are exchanged among agents, without a central server. While both approaches offer privacy benefits by keeping data local, decentralized learning may provide additional privacy benefits by eliminating the central server requirement which is a point of vulnerability that could be targeted in federated learning.

A typical federated or decentralized approach seeks to learn a common model across all devices by leveraging data from a variety of devices. However, in many practical applications, the data generated by different devices is highly non-iid (i.e., non independent and identically distributed), making it sub-optimal to learn a single global model for all devices. To address the statistical heterogeneity challenge,  \emph{personalized learning} can be employed. This approach aims to learn a personalized model for each device based on its own data, while still benefiting from the data of other devices. Various personalization methods have been investigated in the literature~\cite{Nassif2020multitask,Plata-Chaves2017heterogeneous,tan2023towards,kulkarni2020survey}. These strategies include local fine tuning in federated learning, where each client receives a global model and improves it using their local data through several gradient descent steps~\cite{Finn2017Model}. Another approach is contextualization, which involves adapting the model to perform effectively in different contexts~\cite{Wang2017Federated}. Lastly, multitask learning allows for exploiting commonalities and differences across multiple related tasks by learning them jointly~\cite{Smith2017Federated,Nassif2020multitask,Plata-Chaves2017heterogeneous}. 

In this paper, we propose and investigate a novel decentralized multitask learning strategy that allows to jointly learn multiple tasks from streaming data. The approach employs regularization to capture and promote relationships between tasks. Regularization has been previously explored in the context of decentralized collaborative multitask learning. For example, the works~\cite{Nassif2020learning,Chen2014multitask,Koppel2017proximity,Cao2018distributed} employ graph Laplacian regularization to promote task smoothness across the graph. The  work~\cite{Nassif2018regularization} extends~\cite{Nassif2020learning} by employing graph spectral regularization to further improve smoothness promotion. While most works in the literature use smooth graph-based regularizers, only a few explore non-smooth regularizers to promote piecewise constant transitions  between tasks or to perform task clustering.  Examples of such works include~\cite{Nassif2016proximal}, where $\ell_1$-norm based co-regularizers are employed to enforce sparsity on the differences between vectors at neighboring agents, and~\cite{Hallac2015network}, where $\ell_2$-norm based co-regularizers are used to perform automatic clustering over the graph (i.e., grouping nodes that share the same task together). Another family of decentralized multitask learning approaches formulates constrained optimization problems to promote relationships between tasks. These formulations include learning under subspace constraints~\cite{nassif2020adaptation,dilorenzo2020distributed,Chen2014diffusion,korres2010distributed,Kekatos2012distributed,Plata-Chaves2015distributed,Mota2014distributed}, enforcing general equality constraints among neighboring agents~\cite{Nassif2017diffusion,Alghunaim2019distributed}, and potentially extending to non-neighboring agents~\cite{Hua2017penalty}. Application-wise, decentralized multitask learning has proven effective in modeling task relationships, thereby enabling collaborative learning and enhancing global performance across various applications.  These include distributed power system state estimation~\cite{korres2010distributed, Kekatos2012distributed}, distributed weather forecasting~\cite{Nassif2020learning}, distributed housing price prediction~\cite{Hallac2015network}, distributed network flow estimation~\cite{Nassif2017diffusion,Mota2014distributed}, distributed power spectrum estimation in cognitive radio applications~\cite{Plata-Chaves2015distributed,Nassif2016proximal}, among others.

In this work, we aim to develop efficient approaches for networked applications where neighboring devices or agents are expected to assign similar weights to a large subset of features, while a subset of specific  (unknown in advance) features may require different weights to capture localized variations. In distributed healthcare applications, for example, agents (e.g., hospitals or wearable devices) tend to collect the same feature types and number (age, gender, blood pressure, glucose levels, cholesterol levels, etc.). While the number and type of features are the same for all agents, their importance in classification (e.g., predicting the risk of cancer or heart disease) may vary due to regional or behavior differences. For instance, in one region, cholesterol levels might be more significant due to dietary habits, while in another (neighboring) region, blood pressure may carry more weight because of a higher prevalence of hypertension-related problems. However, features such as age, gender, or smoking status might have similar weight across agents if they correlate strongly with the targeted health condition. As we will see, through multitask learning, it is possible to improve overall network performance by exploiting similarities in shared feature weights, while allowing  local adaptation at the level of individual feature weights. 

Given that the weight variability pattern across agents is unknown beforehand, formulating a constrained optimization problem as in~\cite{Plata-Chaves2015distributed,Mota2014distributed,korres2010distributed,Kekatos2012distributed,nassif2020adaptation,dilorenzo2020distributed} to enforce consensus on the shared components of the parameter vectors at neighboring agents is not feasible. Instead, we aim for a solution that promotes sparsity of the differences between parameter vectors at neighboring agents, without prior knowledge of the locations of the different components or weights. We formulate a decentralized multitask optimization problem in which each agent seeks to minimize an individual cost function while incorporating  co-regularization terms to promote collaborative learning with its neighbors. Specifically, each agent optimizes its parameters by minimizing the expected value of a loss function that is defined in terms of its own data. To promote cooperation, non-smooth 
regularization terms are added to the cost function, encouraging sparsity in the parameter differences between neighboring agents. This ensures that models of neighboring agents remain similar or differ only sparsely, enabling collaborative learning while allowing for local variations. We use constant step-sizes $\mu$, instead of diminishing ones, to allow the system to continuously learn and adapt to new data in streaming settings. To handle the non-smooth co-regularization terms, we employ proximal-based methods rather than sub-gradient methods. In deterministic settings, proximal methods are known to be more stable and to converge faster than their sub-gradient counterparts~\cite{combettes2005signal,Wee2013proximal,bertsekas2011incremental}. Under general convexity conditions on the individual costs and co-regularizers, as well as general conditions on the gradient approximations (which are necessary in online settings where agents rely on data-based gradient approximations rather than true gradients), we show that for sufficiently small step-size~$\mu$, the proposed decentralized proximal-based approach converges in the mean-square-error sense within $O(\mu)$ of the solution to the global regularized multitask optimization problem. To improve the computational efficiency of the proposed approach, which requires computing the proximal operator of a weighted sum of sparsity-based norms (such as $\ell_1$-norm, $\ell_0$-norm \footnote{Note that the $\ell_0$-norm is not a true norm, as it does not satisfy the homogeneity property. However, this terminology is widely used in the literature, and we will adopt it accordingly.}, etc.) at each iteration $i$, by each agent~$k$, we provide closed-form expressions for evaluating the proximal operators of the weighted sum of $\ell_1$-norms, $\ell_0$-norms, and elastic net regularization. 
 Simulations are finally conducted to illustrate the theoretical findings and  the effectiveness of the proposed approach.

\noindent\textbf{Notation.} All vectors are column vectors, with random quantities in boldface. Matrices are in uppercase, while vectors and scalars are in lowercase. The operators $\left( \cdot \right)^{\top}$ and $\left( \cdot \right)^{-1}$ denote matrix transpose and inverse, respectively. The operator $\text{col}\{\cdot\}$ stacks column vectors, and $\diag\{\cdot\}$ forms block-diagonal matrices. The $M \times M$ identity matrix is $ I_M$, and  $\preceq$ denotes an element-wise inequality.
\vspace{-0.4 cm}
\section{Problem formulation and decentralized proximal-based approach}
\label{sec: Section 1}
\subsection{Problem formulation}
\label{subsec: problem formulation}
We consider a connected network consisting of $K$ agents, where each agent $k$ is associated with an $M\times 1$ parameter vector $w_k\in\mathbb{R}^M$ and a risk function $J_k(w_k): \mathbb{R}^M \rightarrow \mathbb{R}$, which is assumed to be strongly convex. In most learning and adaptation problems, the risk function $J_k(\cdot)$ is expressed as the expectation of some loss function $ Q_k(\cdot)$ and is written as $J_k(w_k) = \expec Q_k(w_k; \bx_k)$, where $\bx_k$ denotes the random data collected by agent $k$, and the expectation is computed with respect to the data distribution. We denote the unique minimizer of $J_k(w_k)$ by $w_k^o$, so we can write:
\vspace{-0.1 cm}
\begin{equation}
\vspace{-0.2 cm}
\label{eq: local minimizers}
w^o_k\triangleq \arg\min_{w_k}J_k(w_k),
\vspace{-0.2 cm}
\end{equation}
and, consequently,
\vspace{-0.3 cm}
\begin{equation}
\vspace{-0.2 cm}
\cw^o =\arg\min_{\ccw}\sum_{k=1}^{K} J_k(w_k) ,
\end{equation}
where $\cw\triangleq\col \{w_1, ..., w_K\}$ and $\cw^o\triangleq\col \{w_1^o, ..., w_K^o\}$ denote the collection of individual parameter vectors and minimizers, respectively. Recall that the objective in the current work is to derive and study a decentralized strategy that enables collaborative learning in settings where the models or parameter vectors at neighboring agents are expected to differ only sparsely. To achieve this, we propose to incorporate co-regularization terms in order to promote similarity or sparsity in the parameter differences between neighboring agents. That is, we propose to solve the following regularized problem:
\vspace{-0.1 cm}
\begin{equation}
\vspace{-0.1 cm}
\label{eq: global}
\cw_{\eta}^o\triangleq  \arg\min_{\ccw} \left\{ J^{\text{glob}} (\cw) \triangleq \sum_{k=1}^{K} J_k(w_k) + \frac{\eta}{2} R(\cw)\right\},
\vspace{-0.1 cm}
\end{equation} 
where $R(\cw)$ is a regularization term that is assumed to be convex\footnote{Although the convexity condition is classically required for the analysis, we explain in the next section how the proximal operator of a weighted sum of $\ell_0$-norms (which is non-convex) is still worth considering with no computational efforts as its proximal operator can be computed in closed form. Additionally, we compare in the simulation section its performance against convex regularizers, such as those employing $\ell_1$-norms and elastic net regularization, to highlight the advantages of non-convex co-regularization in certain settings.}, but non-differentiable, and is given by:
\vspace{-0.2 cm}
\begin{equation} 
\vspace{-0.1 cm}
R(\cw) \triangleq \sum_{k=1}^{K} \sum_{\ell \in \cN_k} \rho_{k\ell}f_{k\ell}(w_k,w_\ell).
\label{eq: reg}
\vspace{-0.1 cm}
\end{equation} 
where $\cN_k$ denotes the set of neighbors of agent $k$, excluding $k$ itself, $\rho_{k\ell} >0$ is the weight associated with the link $(k,\ell)$ connecting neghboring nodes $k$ and $\ell$, and $f_{k\ell}(w_k,w_{\ell})$ is a non-differentiable regularization function associated with the link  $(k,\ell)$. The weight $\rho_{k\ell}$ allows for local adjustment of the regularization strength, enabling more flexibility in how strongly the parameters $w_k$ and $w_{\ell}$   of neighboring agents are encouraged to be similar or differ sparsely. The regularization strength $\eta>0$ controls the level of cooperation between agents. By setting  $\eta=0$, we obtain the non-cooperative solution in which each agent seeks to minimize its own cost $J_k(w_k)$ without cooperating with neighbors, namely,  $\cw_{\eta}^o=\cw^o$. When $\eta$ approaches infinity, i.e., $\eta \to \infty$, all agents shift their focus towards reaching a global consensus by estimating a common parameter vector $w^o$ defined as the minimizer of the aggregate sum of individual costs:
\vspace{-0.2 cm}
\begin{equation}
\vspace{-0.1 cm}
w^o=\arg\min_w\sum_{k=1}^KJ_k(w).
\vspace{-0.1 cm}
\end{equation} 
The regularization functions $f_{k\ell}(w_k,w_{\ell})$ and $f_{\ell k}(w_{\ell},w_k)$, associated with the links $(k,\ell)$ and  $(\ell,k)$, respectively, are assumed to be symmetric, meaning, $f_{k\ell}(w_k,w_{\ell})=f_{\ell k}(w_\ell,w_k)$. This symmetry ensures that the regularization between neighboring agents remains the same, regardless of the direction of the interaction. Consequently, by summing over all agents in~\eqref{eq: reg}, each term $f_{k\ell}(w_k,w_{\ell})$ can be viewed as weighted by $\frac{\rho_{k\ell} + \rho_{\ell k} }{2}$,  allowing us to write:
\vspace{-0.2 cm}
 \begin{equation} 
 R(\cw) = \sum_{k=1}^{K} \sum_{\ell \in \cN_k}^{} p_{k\ell}f_{k\ell}(w_k,w_\ell),
 \vspace{-0.2 cm}
\end{equation}  
where the factors $\{p_{k\ell}\} $ are symmetric, i.e., $ p_{k\ell} = p_{\ell k}$, and are given by:
\vspace{-0.1 cm}
\begin{equation}
 p_{k\ell}  \triangleq \frac{\rho_{k\ell} + \rho_{\ell k} }{2}.
 \vspace{-0.3 cm}
\end{equation}
\subsection{Decentralized stochastic proximal gradient approach}
\label{subsec: Decentralized stochastic proximal gradient approach}
In decentralized implementations, agent $k$ seeks to estimate the $k$-th sub-vector,  $w_{k,\eta}^o$, of $ \cw_{\eta}^o=\col\left\{w_{1,\eta}^o, ... ,w_{K,\eta}^o\right\}$ in~\eqref{eq: global} by performing local computations and by exchanging updates with its immediate neighbors. On the other hand, in stochastic settings, the distribution of the data $\bx_k$ in $J_k(w_k) = \expec Q_k(w_k; \bx_k)$ is generally unknown. As a result, the true risk and its gradient are unknown. In this case, approximate gradient vectors need to be employed. A common construction in the  stochastic approximation theory is to use the instantaneous approximation based on the available data~\cite{Sayed2014adaptation}:
\vspace{-0.1 cm}
\begin{equation}
\label{eq: approx}
\widehat {\nabla_{w_k} J_k}( w_k) = \nabla_{w_k} Q_k(w_k; \bx_{k,i}),
\vspace{-0.1 cm}
\end{equation}
where $\bx_{k,i}$ represents the data observed by agent $k$ at iteration $i$. The difference between the true gradient and its approximation, known as the \emph{gradient noise}, can be expressed as~\cite{Sayed2014adaptation}:
\vspace{-0.2 cm}
\begin{equation}
\bs_{k,i}(w) \triangleq \nabla_{w_k} J_k( w) - \widehat {\nabla_{w_k} J_k}( w). \label{eq: noise}
\vspace{-0.1 cm}
\end{equation}
In general, it is crucial to ensure that the randomness introduced by this gradient noise does not destabilize the learning process. 

In the following, we propose a stochastic-based approach to solve problem~\eqref{eq: global} in a decentralized manner. To handle the non-smooth regularization, we call upon the forward-backward splitting approach \cite{Combettes2011proximity,Parikh2014proximal}, which is a powerful iterative technique for solving optimization problems of the form:
\vspace{-0.1 cm}
\begin{equation}
\label{eq: prob}
\underset{x \in \mathbb{R}^M}{\min}  f(x) + g(x),
\vspace{-0.1 cm}
\end{equation}
with $f(\cdot)$ is a real-valued, differentiable convex function with a gradient that is $\beta$-Lipschitz continuous, and $g(\cdot)$ is a real-valued, non-differentiable convex function. The proximal gradient method, also known as the forward-backward splitting approach, for solving~\eqref{eq: prob} is given by~\cite{Combettes2011proximity,Parikh2014proximal}: 
\vspace{-0.1 cm}
\begin{equation}
\label{eq: iteration}
x_{i+1}= \prox_{\gamma g}\left(x_i - \gamma \nabla_x f(x_i) \right),
\vspace{-0.1 cm}
\end{equation}
where $\gamma$ is a constant step-size chosen such that $\gamma \in \left(0,2\beta^{-1}\right)$ to ensure convergence to the minimizer of \eqref{eq: prob}, and $\prox_{{\gamma} g}(v)$ is the \emph{proximal} operator of $\gamma g(x)$ at a given point $ v \in \mathbb{R}^M$ and is defined as:
\vspace{-0.1 cm}
\begin {equation}
\prox_{{\gamma} g}(v) =\arg\min_{x} g(x) + \frac{1}{2 \gamma} \Vert x - v \Vert ^2 .    \label{eq: prox}
\vspace{-0.1 cm}
\end{equation}
Note that it is also possible to apply the approach in~\eqref{eq: iteration} to scenarios where the function $g(\cdot)$ is non-convex. However, in this  case, the solution to  the optimization problem in~\eqref{eq: prox} is no longer single-valued~\cite{attouch2013convergence}.
The gradient-descent step in~\eqref{eq: iteration} is the \emph{forward} step (explicit step), and the proximal step is the \emph{backward} step (implicit step). By considering the regularized problem~\eqref{eq: global}, it can be observed that aggregate sum of individual costs corresponds to a smooth function, and that a stochastic gradient step on this smooth function would take the following form:
\vspace{-0.2 cm}
\begin{equation}
\bpsi_i=\col\left\{\bw_{k,i-1}-\mu\widehat{\nabla_{w_k} J_k}( \bw_{k,i-1}) \right\}_{k=1}^K,
\vspace{-0.1 cm}
\end{equation}
where $\bw_{k,i-1}$ represents the current weight vector at agent $k$, $\mu$ is a small positive step-size parameter, and  $\widehat{\nabla_{w_k} J_k}( \bw_{k,i-1})$ is an instantaneous approximation of the gradient vector at agent $k$. The non-smooth function in~\eqref{eq: global} corresponds to~$\eta R(\cw)$, so that a forward-backward splitting-based approach for solving~\eqref{eq: global} would take the following form:
\vspace{-0.1 cm}
\begin{equation}
\label{eq: centralized proximal approach}
\bcw_i=\prox_{\mu\eta R}\left(\col\left\{\bw_{k,i-1}-\mu\widehat{\nabla_{w_k} J_k}( \bw_{k,i-1}) \right\}_{k=1}^K\right),
\vspace{-0.1 cm}
\end{equation}
where $\bcw_i=\col\{\bw_{1,i},\ldots,\bw_{K,i}\}$ is the estimate of  $\cw^o_{\eta}$  at iteration $i$. While the gradient step in~\eqref{eq: centralized proximal approach} can be implemented in a fully decentralized manner, the proximal step requires the computation of the proximal operator of $\mu\eta R(\cw)$, which cannot be performed in a decentralized manner. In the following, we propose an alternative decentralized solution that allows each agent to handle its own regularization independently, thereby enabling the computation of  the proximal operator locally at each agent. Then, in Sec.~\ref{sec: Section 2}, we study its mean-square-error stability  w.r.t. the solution $\cw^o_{\eta}$ of the global regularized cost in~\eqref{eq: global}.

In this work, we investigate the following decentralized stochastic proximal gradient approach for solving~\eqref{eq: global}:
\vspace{-0.1 cm}
\begin{subequations} 
 \label{eq: algor}
 \begin{empheq}[left={\empheqlbrace\,}]{align}
    \bpsi_{k,i} &= \bw_{k,i-1} - \mu \widehat {\nabla_{w_k} J_k}( \bw_{k,i-1}) \label{eq: algor step 1}  \\                          
    \bw_{k,i}   & = \prox_{\mu\eta \boldsymbol{g}_{k,i}} ( \bpsi_{k,i} )       \label{eq: algor step 2}                                                                 
 \end{empheq}
\end{subequations}
where $\boldsymbol{g}_{k,i}:\mathbb{R}^M\rightarrow \mathbb{R}$ is a non-smooth regularization function given by: 
\vspace{-0.4 cm}
\begin{equation}
\label{eq: gpsi}
\boldsymbol{g}_{k,i}(w_k)= \sum_{ \ell \in \cN_k} p_{k\ell} f_{k\ell}(w_k, \bpsi_{\ell,i}).
\vspace{-0.1 cm}
\end{equation}
The first step~\eqref{eq: algor step 1}, known as the \emph{self-learning} step~\cite{Nassif2020multitask}, corresponds to the stochastic gradient step on the individual cost $ J_k(w_k)$, where $ \mu>0$ is a  small step-size parameter. The result of step~\eqref{eq: algor step 1} is $ \bpsi_{k,i}$, the intermediate estimate of $ w^o_{k,\eta} $ at iteration $i$, agent $k$. In the subsequent step~\eqref{eq: algor step 2}, known as the \emph{social-learning} step~\cite{Nassif2020multitask}, agent~$k$ collects the intermediate estimates $\{\bpsi_{\ell,i}\}$ from its neighbors $\ell\in\cN_k$ and implements the proximal operator of the function $\mu \eta\boldsymbol{g}_{k,i}(w_k)$ defined in~\eqref{eq: gpsi}. This step enables collaborative learning by leveraging information exchange among neighboring agents. 
Note that the proximal operator of the function $\mu \eta \boldsymbol{g}_{k,i}(w_k)$  must be evaluated at every iteration~$i$, for all agents $k$ in the network. Therefore, having a closed-form expression for its computation is important to achieve higher computational efficiency. In Sec.~\ref{sec: Section 3}, we provide closed-form expressions for the proximal operator of a weighted sum of commonly used sparsity-based norms.

\subsection{Related work and contributions}
\label{subsec: Related work and contributions}
The current work is not the first to investigate non-smooth regularization in the context of decentralized multitask learning. For instance, the work~\cite{Nassif2016proximal} employs $\ell_1$-norm based co-regularizers to promote sparsity in the differences between vectors of neighboring agents and employs a  decentralized stochastic proximal gradient method similar to that in~\eqref{eq: algor}. In the same context, the work~\cite{jin2020online} investigates the  $\ell_{\infty,1}$-norm to encourage joint sparsity through co-regularization. The work~\cite{Hallac2015network} proposes a primal-dual based approach to solve regularized problems of the form~\eqref{eq: global} where $\ell_2$-norm based co-regularizers are used to encourage automatic clustering over the graph (i.e., grouping agents that share the same task together). Unlike the current work, which focuses on stochastic optimization, the work~\cite{Hallac2015network} addresses deterministic optimization. Compared to the previous works~\cite{Nassif2016proximal,jin2020online}, our contributions can be summarized as follows:
\begin{itemize}
\item First, unlike these prior works which focus on mean-square-error costs of the form $J_k(w_k)=\expec(\bd_k(i)-\bu_{k,i}^\top w_k)^2$, our work generalizes the non-smooth regularization framework over multitask networks and applies it to a broader class of individual costs, such as logistic regression costs -- see Example~\ref{ex: lr} further ahead. This generalization enables the solution of a wide range of classification and machine learning problems.
\item Second, our analysis focuses on evaluating how effectively the multitask strategy~\eqref{eq: algor} approaches the optimal solution $\cw_{\eta}^o$ of the regularized cost~\eqref{eq: global}. In contrast, prior works~\cite{Nassif2016proximal,jin2020online} focus on showing how collaborative learning affects the mean-square-error stability (studied w.r.t. the local minimizers $w^o_k$ in~\eqref{eq: local minimizers}). While such analyses are valuable, they do not address how effectively the decentralized approach solves the global problem~\eqref{eq: global}.
\item Finally, with the aim of improving the generality and broader applicability of collaborative learning in non-smooth regularization contexts, and following ideas developed in~\cite{Nassif2016proximal} for deriving closed-form expression for the proximal operator of a weighted sum of $\ell_1$-norms, i.e., when $\boldsymbol{g}_{k,i}(w_k)= \sum_{ \ell \in \cN_k} p_{k\ell} \|w_k-\bpsi_{\ell,i}\|_1$, we derive closed-form expressions for computing the proximal operator of the function $\boldsymbol{g}_{k,i}(w_k)$ in two specific cases: \emph{i)} when $\boldsymbol{g}_{k,i}(w_k)$ is a weighted sum of $\ell_0$-norms, i.e., $\boldsymbol{g}_{k,i}(w_k)= \sum_{ \ell \in \cN_k} p_{k\ell} \|w_k-\bpsi_{\ell,i}\|_0$, and \emph{ii)} when  $\boldsymbol{g}_{k,i}(w_k)$ is a weighted sum of elastic net regularizations, i.e., $\boldsymbol{g}_{k,i}(w_k)= \sum_{ \ell \in \cN_k} p_{k\ell}\left( \|w_k-\bpsi_{\ell,i}\|_1+\frac{\beta}{2} \|w_k-\bpsi_{\ell,i}\|_2^2\right)$. 
\end{itemize}


\section{Network mean-square-error stability}
\label{sec: Section 2}

In this section, we study how effectively the multitask strategy \eqref{eq: algor} approaches the optimal solution $\cw_{\eta}^o$ of the regularized cost~\eqref{eq: global} by analyzing the mean-square-error $\expec \Vert w_{k,\eta}^o - \bw_{k,i} \Vert^2$. Before proceeding, we introduce the following assumptions on the individual costs $J_k(w_k)$, the gradient noise processes $\bs_{k,i}(w)$ defined in~\eqref{eq: noise}, and the non-smooth co-regularization functions $\bg_{k,i}(w_k)$ defined in~\eqref{eq: gpsi}. 

\begin{assumption} (Conditions on the individual costs). \label{ass: boundedhess}
For each agent $k$, the individual cost $J_k(w_k)$ is assumed to be twice-differentiable and strongly convex. Moreover, the Hessian matrix function $H_k(w_k)= \nabla_{w_k} ^2J_k(w_k)$ is assumed to satisfy the following condition:
\vspace{-0.2 cm}
\begin{equation}
0 < \nu_k I_M \leq H_k(w_k) \leq \delta_k I_M,
\vspace{-0.2 cm}
\end{equation}
for some $0<\nu_k\leq \delta_k$.\hfill\qed
\end{assumption}
\begin{assumption} (Conditions on the gradient noise).
\label{ass: gradient noise}
For each agent $k$, the gradient noise process defined in~\eqref{eq: noise} is assumed to satisfy:
\vspace{-0.4 cm}
\begin{align}
\expec\left[\bs_{k,i}(\bw) \mid {\bcF}_{i-1}\right] &= 0,\\
\expec\left[\| \bs_{k,i}(\bw)\|^2 \mid {\bcF}_{i-1}\right] &\leq \beta_{s,k}^2 \| \bw \|^2 + \sigma_{s,k}^2,
\vspace{-0.2 cm}
\end{align}
for any $\bw\in\bcF_{i-1}$, for some $\beta_{s,k}^2 \geq 0$ and  $\sigma_{s,k}^2 \geq 0$, and where ${\bcF}_{i-1}$ denotes the filtration generated by the random processes $\{\bw_{\ell,j}\} $ for all $ \ell=1,...,K$ and $ j \leq i-1$.\hfill\qed
\end{assumption}
 \begin{assumption}(Conditions on the non-smooth co-regularizers).
 \label{ass: ass3}
With each agent $k$, we associate a non-smooth convex function given by:
\vspace{-0.1 cm}
 \begin{equation}
 \label{eq: definition of the function rk}
 r_k\left(w_k,\{w_{\ell}\}_{\ell\in\cN_k}\right)\triangleq\sum_{\ell\in\cN_k}p_{k\ell}f_{k\ell}(w_k,w_{\ell}),
 \vspace{-0.1 cm}
 \end{equation}
 where $\cN_k$ refers to the set of neighbors of agent $k$, excluding $k$ itself. 
 We assume that the subdifferential set of~\eqref{eq: definition of the function rk} denoted by $\partial_{w_k} r_k\left(w_k,\{w_{\ell}\}_{\ell\in\cN_k}\right)$ is uniformly bounded, namely,
   \begin{equation}
 \label{boundedsub}
\left\| v \right\| \leq e_k, \qquad \forall v \in   \partial_{w_k} r_k\left(w_k,\{w_{\ell}\}_{\ell\in\cN_k}\right),  w_k,w_{\ell}\in\mathbb{R}^M
 \end{equation}
where $e_k$ is some non-negative constant. 
\hfill\qed
 \end{assumption}
 \noindent
As explained in~\cite{Sayed2014adaptation}, the conditions in Assumptions~\ref{ass: boundedhess} and~\ref{ass: gradient noise} are satisfied by many objective functions of interest in learning and adaptation such as quadratic and logistic risks. On the other hand, as we show in Appendix~\ref{Appendix: A}, Assumption~\ref{ass: ass3} is satisfied by the $\ell_1$-norm co-regularization, i.e., when the functions $f_{k\ell}(w_k,w_{\ell})$ in~\eqref{eq: definition of the function rk} are of the form $\|w_k-w_{\ell}\|_1$. 
Before presenting the theorem that quantifies how closely the stochastic algorithm \eqref{eq: algor} approaches $ \cw_{\eta}^o$ in~\eqref{eq: global}, we first define the mean-square perturbation vector at time $i$ relative to $ \cw_{\eta}^o$: 
\vspace{-0.1 cm}
  \begin{equation}
 \text{ MSP}_i \triangleq \col \left \{ \expec \vert \vert  \bw_{k,i} - w_{k,\eta}^o   \vert \vert^2 \right \}_{k=1}^K.\label{eq: definition of MSP}
 \vspace{-0.1 cm}
  \end{equation}
  The $k$-th entry of $\text{MSP}_i $ characterizes how far away the estimate $\bw_{k,i}$ at agent $k$ and time $i$ is from $w_{k,\eta}^o$, the $k$-th subvector of $ \cw_{\eta}^o$. 
  \begin{theorem} \label{theorem2}(Network mean-square-error stability). Under Assumptions \ref{ass: boundedhess}, \ref{ass: gradient noise} and \ref{ass: ass3}, and for a regularization strength~$\eta$ of the form $\eta= \kappa \mu^{\alpha}$, where $\alpha \geq \frac{1}{2}$ and $\kappa$ is some positive constant independent of $\mu$, the mean-square perturbation can be recursively bounded as: 
  \vspace{-0.1 cm}
  \begin{equation}
  \label{eq: recursion of the MSP}
  \emph{ \text{MSP}}_i \preceq A \emph{\text{ MSP}}_{i-1} + \mu (c + \mu d), 
  \vspace{-0.1 cm}
  \end{equation}
  where: 
  \vspace{-0.4 cm}
\begin{align}
A &\triangleq {\emph{\diag }}\left \{ 1 - \frac{\mu \nu_k}{2} + \frac{ 2 \mu^2  \beta_{s,k}^2}{1 -  \frac{\mu \nu_k}{2}} \right \}_{k=1}^{K},\label{eq: A definition}\\
c& \triangleq  \eta^2{\emph{\col }} \left \{ \frac{8 e_k^2}{\nu_k} \right \}_{k=1}^{K}, \label{eq: c}\\
d &\triangleq{\emph{\col }} \left \{ \frac{ 2  \beta_{s,k}^2}{1 -  \frac{\mu \nu_k}{2}}  \vert \vert w_{k,\eta}^o \vert \vert^2 + \frac{ \sigma_{s,k} ^2 }{1 -  \frac{\mu \nu_k}{2}}  \right \}_{k=1}^{K}.  \label{eq: d}
\vspace{-0.4 cm}
\end{align}  
A sufficient condition for ensuring the stability of recursion~\eqref{eq: recursion of the MSP} is: 
\vspace{-0.4 cm}
  \begin{equation}
  0 < \mu < \underset{ 1 \leq k \leq K}{\min}\left \{ \frac{\nu_k}{\delta_k^2},  \frac{\nu_k}{4 \beta_{s,k}^2 + \frac{\nu_k^2}{2}} \right \} .
  \vspace{-0.2 cm}
  \end{equation}
  \noindent
  It then follows that: 
  \vspace{-0.2 cm}
  \begin{equation}
  \label{eq: result}
 \left\| \limsup_{i \to \infty} \emph{\text{ MSP}}_i \right\|_{\infty}  = O(\mu) ,
 \vspace{-0.1 cm}
  \end{equation}
  where $\|x\|_{\infty}$ is the $\ell_{\infty}$-norm of the vector $x$ defined as $\|x\|_{\infty}=\max_{m}|x_m|$ with $x_m$ denoting the $m$-th entry of~$x$.
  \end{theorem}
\begin{proof} 
By examining the optimality condition of~\eqref{eq: prox}, which states that the $0$ vector belongs to the subdifferential of $g(x)+\frac{1}{2\gamma}\|x-v\|^2$ at the minimizer $\prox_{\gamma g}(v)$~\cite{Parikh2014proximal}, and applying this condition to step~\eqref{eq: algor step 2} of our approach, we obtain: 
\vspace{-0.2 cm}
\begin{equation}
0 \in \partial_{w_k} \bg_{k,i} ( \bw_{k,i} ) + \frac{1}{\mu\eta}(\bw_{k,i} - \bpsi_{k,i}).
\vspace{-0.1 cm}
\end{equation}
From step~\eqref{eq: algor step 1}, we can write the following relation:
\vspace{-0.2 cm}
\begin{equation}
0 \in \partial_{w_k} \bg_{k,i} ( \bw_{k,i} ) + \frac{1}{\mu\eta}\left(\bw_{k,i} - \bw_{k,i-1} + \mu \widehat{\nabla_{w_k} J_k}( \bw_{k,i-1}) \right),
\vspace{-0.2 cm}
\end{equation}
which can be re-written as:
\vspace{-0.1 cm}
\begin{equation}
\frac{1}{\mu\eta}\left(\bw_{k,i-1} - \bw_{k,i} - \mu \widehat{\nabla_{w_k} J_k}( \bw_{k,i-1}) \right) \in \partial_{w_k} \bg_{k,i} ( \bw_{k,i} )  .
\vspace{-0.1 cm}
\end{equation}
We denote $\widehat  {\partial_{w_k} \bg_{k,i} }( \bw_{k,i} )$ as the particular subgradient selected by the proximal operator from $\partial_{w_k} \bg_{k,i}( \bw_{k,i} )$, and write:
\vspace{-0.2 cm}
\begin{equation}
\widehat{\partial_{w_k}\bg_{k,i}}( \bw_{k,i} ) = \frac{1}{\mu\eta}\left(\bw_{k,i-1} - \bw_{k,i} - \mu \widehat{\nabla_{w_k} J_k}( \bw_{k,i-1}) \right) ,
\vspace{-0.2 cm}
\end{equation}
from which we conclude that\footnote{From~\eqref{eq: subg}, we observe that the gradient $ \widehat{\nabla_{w_k} J_k}(\cdot) $ is evaluated at $ \bw_{k,i-1}$, whereas the subgradient $\partial_{w_k}\bg_{k,i} (\cdot)$ is computed at $\bw_{k,i}$. This is why the proximal gradient algorithm is often referred to as the ``forward-backward'' approach.}:
\vspace{-0.2 cm}
\begin{equation}
\label{eq: subg}
\bw_{k,i} = \bw_{k,i-1} - \mu \widehat{ \nabla_{w_k} J_k}( \bw_{k,i-1}) - \mu\eta \widehat{\partial_{w_k}\bg_{k,i}}( \bw_{k,i} ) .
\end{equation}

Now, let $\bwt_{k,i}\triangleq\bw_{k,i}-w^o_{k,\eta}$. By subtracting $w_{k,\eta}^o$  from both sides of~\eqref{eq: subg}, we obtain:
\vspace{-0.1 cm}
 \begin{align}
\bwt_{k,i}&= \bwt_{k,i-1} - \mu \widehat{ \nabla_{w_k} J_k}( \bw_{k,i-1}) - \mu\eta \widehat{\partial_{w_k}\bg_{k,i}}( \bw_{k,i} )\notag\\
&\overset{\eqref{eq: noise}}=\bwt_{k,i-1} - \mu  \nabla_{w_k} J_k( \bw_{k,i-1}) + \mu  \bs_{k,i}(\boldsymbol{w}_{k,i-1}) \notag\\
& \quad - \mu\eta \widehat{\partial_{w_k}\bg_{k,i}}( \bw_{k,i} ). \label{eq: relation1}
\end{align}
By applying the mean-value theorem to the twice differentiable function $J_k(\cdot)$~\cite{Polyak1987introduction}, we can write:  
\begin{equation}
\label{eq: mean-value theorem}
 \nabla_{w_k} J_k( \bw_{k,i-1})= \nabla_{w_k} J_k( w^o_{k,\eta})+ \bH_{k,i-1}\bwt_{k,i-1},
\end{equation}
\vspace{-0.2 cm}
where:
\begin{equation}
 \bH_{k,i-1} \triangleq \int_{0}^{1} \nabla_{w_k} ^2 J_k( w_{k,\eta}^o + t \bwt_{k,i-1})dt.
 \end{equation}
 By using~\eqref{eq: mean-value theorem} into~\eqref{eq: relation1}, we conclude that:
\vspace{-0.1 cm}
\begin{equation}
\label{eq: relation2}
\boxed{
\begin{aligned}
\bwt_{k,i} = & \left( I_{M} - \mu \bH_{k,i-1} \right) \bwt_{k,i-1} 
 - \mu \eta \widehat{\partial_{w_k} \bg_{k,i}}( \bw_{k,i} ) \\
& + \mu \bs_{k,i}(\bw_{k,i-1}) - \mu b_k
\end{aligned}
}
\end{equation}
where $b_k$ is an $M\times 1$ vector defined as:
\vspace{-0.1 cm}
 \begin{equation}
 b_k\triangleq\nabla_{w_k} J_k\left( w_{k,\eta}^o\right).\label{eq: bias vector}
 \end{equation}
 
Now, by taking the squared norm of both sides of~\eqref{eq: relation3} and then computing expectations, we obtain: 
\begin{align}
\expec \|  \bwt_{k,i}\|^2 
&= \expec \big\| \big( I_{M} - \mu \bH_{k,i-1} \big) \bwt_{k,i-1} 
+ \mu \bs_{k,i}(\bw_{k,i-1})  \notag \\
& \quad - \mu\eta \widehat{\partial_{w_k}\bg_{k,i}}( \bw_{k,i} ) 
- \mu b_k \big\|^2  \notag \\
&\overset{\text{(a)}}{\leq} \frac{1}{1 - t} 
\expec \big\|  \big( I_{M} - \mu \bH_{k,i-1} \big) \bwt_{k,i-1} \notag \\
& + \mu \bs_{k,i}(\bw_{k,i-1}) \big\|^2  
 + \frac{\mu^2}{t} \big\| 
\eta \widehat{\partial_{w_k} \bg_{k,i}}( \bw_{k,i} ) + b_k \big\|^2.  \label{eq: intermediate MSP}  
\end{align}
 where in step (a) we applied Jensen's inequality to the convex function $\|\cdot\|^2$ for any arbitrary positive number $t\in(0,1)$~\cite{Jensen1906fonctions}.

In the following, we upper bound the first expectation on the RHS of inequality~\eqref{eq: intermediate MSP}. By using Assumption~\ref{ass: gradient noise} and the sub-multiplicative property of norms, we can write: 
\vspace{-0.1 cm}
  \begin{align}
& \expec \left[\vert \vert    (I_{M} - \mu \bH_{k,i-1})\bwt_{k,i-1} +   \mu \bs_{k,i}(\bw_{k,i-1})     \vert\vert ^2 |\bcF_{i-1}\right] \notag\\
 &= \|(I_{M} - \mu \bH_{k,i-1})\bwt_{k,i-1}\|^2  \notag\\
 & \qquad +\mu^2\expec \left[\|\bs_{k,i}(\bw_{k,i-1}) \|^2|\bcF_{i-1} \right]\notag\\
 &\leq\|I_{M} - \mu \bH_{k,i-1}\|^2\|\bwt_{k,i-1}\|^2 \notag\\
 &\qquad +\mu^2\expec \left[\|\bs_{k,i}(\bw_{k,i-1}) \|^2|\bcF_{i-1} \right].\label{eq: first term in MSP}
 \end{align} 
 \vspace{-0.2 cm}
From Assumption~\ref{ass: boundedhess}, we can write: 
\begin{equation}
  (1-\mu\delta_k)I_{M}    \leq I_{M} - \mu \bH_{k,i-1} \leq    (1-\mu\nu_k )I_{M}.
\end{equation}
\vspace{-0.2 cm}
 Consequently, $\|I_{M} - \mu \bH_{k,i-1}\|^2$ can be upper bounded as:
 \begin{align}
\Vert I_{M} - \mu \bH_{k,i-1}\Vert ^2 &\overset{(\text{a})}= \left[ \rho( I_{M} - \mu \bH_{k,i-1}) \right] ^2\notag\\ 
&\leq \max \left \{ (1-\mu \nu_k)^2, (1 - \mu\delta_k)^2 \right \}\notag\\
&\leq 1 - 2 \mu \nu_k + \mu^2 \delta_k^2\notag\\
&\overset{(\text{b})}\leq 1 -  \mu \nu_k \notag\\
&\leq \left(1- \frac{\mu \nu_k }{2}\right)^2 ,\label{eq: upper bound on i-muh}
\end{align}
where in (a) we used the fact that $I_{M} - \mu \bH_{k,i-1}$ is  symmetric, so that its $2$-induced norm is equal to its spectral radius, and in (b) we assumed that $\mu$ is sufficiently small and satisfies:
\vspace{-0.1 cm}
\begin{equation}
\mu <\underset{ 1 \leq k \leq K}{\min}\left \{ \frac{\nu_k}{\delta_k^2} \right \}.
\end{equation} 
From Assumption~\ref{ass: gradient noise}, we can upper bound the gradient noise term in~\eqref{eq: first term in MSP} according to:
\vspace{-0.2 cm}
 \begin{align}
   \expec \left[\vert \vert \bs_{k,i}(\bw_{k,i-1}) \vert\vert ^2 \mid {\bcF}_{i-1} \right]  &\leq \beta_{s,k}^2 \|\bw_{k,i-1}\|^2  + \sigma_{s,k} ^2 \notag\\
  & = \beta_{s,k}^2 \lVert  \bw_{k,i-1} - w_{k,\eta}^o + w_{k,\eta}^o  \rVert ^2 \notag\\
  & \quad +  \sigma_{s,k} ^2\notag\\
  &  \leq 2 \beta_{s,k}^2 \lVert  \bwt_{k,i-1} \rVert ^2   \notag\\
  & \quad + 2 \beta_{s,k}^2 \lVert w_{k,\eta}^o  \rVert ^2 +   \sigma_{s,k} ^2 , \label{eq: expecnoise}
\end{align} 
where in the last step we applied Jensen's inequality to the convex function $\|\cdot\|^2$. By using inequalities~\eqref{eq: upper bound on i-muh} and~\eqref{eq: expecnoise} into~\eqref{eq: first term in MSP}, we obtain:
\vspace{-0.1 cm}
 \begin{align}
& \expec \left[\vert \vert    (I_{M} - \mu \bH_{k,i-1})\bwt_{k,i-1} +   \mu \bs_{k,i}(\bw_{k,i-1})     \vert\vert ^2 |\bcF_{i-1}\right]\notag\\
&\leq\left(\left(1- \frac{\mu \nu_k }{2}\right)^2+2\mu^2 \beta_{s,k}^2\right)\|\bwt_{k,i-1}\|^2 \notag\\
&\qquad + 2\mu^2 \beta_{s,k}^2 \lVert w_{k,\eta}^o  \rVert ^2 + \mu^2  \sigma_{s,k} ^2.\label{eq: first term in MSP}
 \end{align} 

Let us now upper bound the second term on the RHS of inequality~\eqref{eq: intermediate MSP}. By applying Jensen's inequality to the convex function $\|\cdot\|^2$, we can write:
\vspace{-0.2 cm}
\begin{align}
\left\|\eta \widehat{\partial_{w_k}\bg_{k,i}}( \bw_{k,i} )+b_k\right\|^2&\leq 2\eta^2\left\|\widehat{\partial_{w_k}\bg_{k,i}}( \bw_{k,i} )\right\|^2+2\left\|b_k\right\|^2\notag\\
&\overset{(\text{a})}\leq2\eta^2e_k^2+2\left\|b_k\right\|^2,\label{eq: relation5}
\vspace{-0.2 cm}
\end{align}
where step (a) follows from Assumption~\ref{ass: ass3} and from the fact that $\bg_{k,i}(w_k)$ in~\eqref{eq: gpsi} is equal to  $r_{k}(w_k,\{w_{\ell}\}_{\ell\in\cN_k})$ in~\eqref{eq: definition of the function rk}, with the variables $w_{\ell}$ replaced by the intermediate estimates $\bpsi_{\ell,i}$. To bound the norm of the vector $b_k$ defined in~\eqref{eq: bias vector}, we use the optimality  optimality conditions of~\eqref{eq: global}, and the fact that the individual costs $J_k(w_k)$ are strongly convex and the non-smooth regularizer $R(\cw)$ is convex, to write: 
\vspace{-0.1 cm}
\begin{equation}
0 \in  \col \left\{ \nabla_{w_k} J_k( w_{k,\eta}^o)\right\}_{k=1}^K + \frac{\eta}{2} \partial_{\ccw}  R(\cw_{\eta}^o),
\end{equation}
from which we conclude that:
\begin{equation}
-\col \left\{ b_k\right\}_{k=1}^K \in \frac{\eta}{2} \partial_{\ccw}  R(\cw_{\eta}^o).
\end{equation}
Since the regularization functions $f_{k\ell}(w_k,w_{\ell})$ and $f_{\ell k}(w_{\ell},w_k)$ are assumed to be symmetric, meaning, $f_{k\ell}(w_k,w_{\ell})=f_{\ell k}(w_\ell,w_k)$, and the factors $\{p_{k\ell}\} $ are symmetric, i.e., $ p_{k\ell} = p_{\ell k}$, we can write:
\begin{equation}
\partial_{w_k}  R(\cw)=2\partial_{w_k} r_k\left(w_k,\{w_{\ell}\}_{\ell\in\cN_k}\right),
\vspace{-0.1 cm}
\end{equation}
and, consequently, at the solution $\cw^o_{\eta}$, the sub-differential set of  $R(\cw)$ w.r.t. $w_k$ satisfies: 
\begin{equation}
\label{eq: relation3}
-b_k\in\eta\partial_{w_k} r_k\left(w^o_{k,\eta},\{w^o_{\ell,\eta}\}_{\ell\in\cN_k}\right).
\end{equation}
From~\eqref{eq: relation3} and Assumption~\ref{ass: ass3}, we can then conclude that:
\begin{equation}
\label{eq: relation4}
\|b_k\|^2\leq \eta^2e_k^2.
\end{equation}
By using the bound~\eqref{eq: relation4} into~\eqref{eq: relation5}, we can write: 
\begin{equation}
\vspace{-0.1 cm}
\left\|\eta \widehat{\partial_{w_k}\bg_{k,i}}( \bw_{k,i} )+b_k\right\|^2\leq 4\eta^2e_k^2.\label{eq: relation6}
\end{equation}

By using~\eqref{eq: first term in MSP} and~\eqref{eq: relation6} into~\eqref{eq: intermediate MSP}, and by selecting $ t = \frac{\mu \nu_k}{2}$ (which is guaranteed to be in $(0,1)$ for sufficiently small step-size $\mu$), we can finally write: 
 \begin{align}
 \vspace{-0.2 cm}
 \expec \|  \bwt_{k,i}\|^2&\leq \left(1 - \frac{\mu \nu_k}{2} +  \frac{ 2 \mu^2  \beta_{s,k}^2}{1 -  \frac{\mu \nu_k}{2}}  \right) \expec  \Vert  \bwt_{k,i-1} \Vert ^2 \notag\\
 & +  \frac{ 2 \mu^2  \beta_{s,k}^2}{1 -  \frac{\mu \nu_k}{2}}  \vert \vert w_{k,\eta}^o \vert \vert^2 + \frac{ \mu^2 \sigma_{s,k} ^2 }{1 -  \frac{\mu \nu_k}{2}} + \mu\eta^2\frac{8e_k^2}{\nu_k}.\label{eq: intermediate MSP new}
 \vspace{-0.2 cm}
 \end{align}
It can be shown that the mean-square perturbation vector at time $i$, which is defined in~\eqref{eq: definition of MSP}, can be bounded according to~\eqref{eq: recursion of the MSP}. 
 By iterating~\eqref{eq: recursion of the MSP} starting from $i=1$, we get: 
 \vspace{-0.2 cm}
  \begin{equation}
 \text{MSP}_i \preceq A^{i} \text{ MSP}_{0} +\mu\sum_{j=0}^{i-1} A^j (c + \mu d).
 \vspace{-0.1 cm}
 \end{equation}  
The matrix $A$ in~\eqref{eq: A definition} can be guaranteed to be stable, i.e., its spectral radius, $\rho(A)$, is less than one. To see this, we start by noting that the spectral radius of a matrix is upper bounded by any of its induced norms, namely,
  \begin{equation}
  \vspace{-0.1 cm}
  \rho (A) \leq \Vert A \Vert_{\infty} = \underset {1 \leq k \leq K}{\max} \left\{ 1 - \frac{\mu \nu_k}{2} + \frac{ 2 \mu^2  \beta_{s,k}^2}{1 -  \frac{\mu \nu_k}{2}}  \right\},\label{eq: infinity norm of A}
  \end{equation}
 where $\Vert A \Vert_{\infty}$ is the induced $\ell_{\infty}$-norm of $A$ given by the maximum of the absolute row sums. Then, to guarantee the stability of $A$, we can choose the step-size $\mu$ according to:
 \vspace{-0.1 cm}
  \begin{equation}
  \vspace{-0.2 cm}
  \label{eq: cond}
  0 < \mu <  \underset {1 \leq k \leq K}{\min} \left \{ \frac{\nu_k}{4 \beta_{s,k}^2 + \frac{\nu_k^2}{2}} \right \}.
  \end{equation}
  In this case, we have: 
  \vspace{-0.3 cm}
  \begin{equation}
  \limsup_{i \to \infty} \text{ MSP}_i \preceq  \mu \sum_{j=0}^{\infty} A^{j } ( c + \mu d).
  \vspace{-0.2 cm}
  \end{equation}
Using the submultiplicative and subadditivity properties of the induced infinity norm, we obtain:   
\vspace{-0.2 cm}
  \begin{align}
\left\| \limsup_{i \to \infty} \text{MSP}_i\right\|_{\infty} &\leq  \mu \left\|\sum_{j=0}^{\infty} A^{j }\right\|_{\infty} \left\| c + \mu d\right\|_{\infty} \notag\\
& \leq  \mu  \sum_{j=0}^{\infty}\left\| A^{j }\right\|_{\infty} \lVert c + \mu d\rVert_{\infty} \notag\\
& \leq  \mu  \sum_{j=0}^{\infty}\left\| A\right\|^j_{\infty} \lVert c + \mu d\rVert_{\infty} \notag\\
&= \mu \frac{\lVert c + \mu d\rVert_{\infty} }{1 - \lVert A \rVert_{\infty}},  \label{eq: subprop}
\end{align}
where we used the fact that $ \Vert A \Vert_{\infty} <1$ under condition~\eqref{eq: cond}. Now, let us quantify $1 - \lVert A \rVert_{\infty}$. To that end, let:
\begin{equation}
  \zeta_k \triangleq \frac{ \nu_k}{2} - \frac{ 2 \mu  \beta_{s,k}^2}{1 -  \frac{\mu \nu_k}{2}},
  \end{equation}
which is guaranteed to be in $(0,1)$ under condition~\eqref{eq: cond}.  From~\eqref{eq: infinity norm of A}, we can write:
  \begin{equation}
   \Vert A \Vert_{\infty} = \underset {1 \leq k \leq K}{\max} \left\{ 1 - \mu \zeta_k \right\} = 1 - \mu  \underset {1 \leq k \leq K}{\min} \zeta_k.
  \end{equation}
 Substituting into~\eqref{eq: subprop}, we obtain: 
  \begin{equation}
  \vspace{-0.1 cm}
\left\| \limsup_{i \to \infty} \text{ MSP}_i\right\|_{\infty} \leq  \frac{\Vert c + \mu d\Vert_{\infty} }{ \underset {1 \leq k \leq K}{\min} \zeta_k}.
\vspace{-0.1 cm}
  \end{equation}
From definitions~\eqref{eq: c} and~\eqref{eq: d}, we have $ \Vert \mu d \Vert_{\infty} = O(\mu)$ and $\Vert c \Vert_{\infty} = O(\mu^{2\alpha})$ when $\eta=\kappa\mu^{\alpha}$. Now, since $\Vert c + \mu d\Vert_{\infty} \leq~\Vert c\Vert_{\infty} + \Vert\mu d\Vert_{\infty} $, we can conclude that~\eqref{eq: result} holds when $\alpha \geq \frac{1}{2}$, which completes the proof.
\end{proof}
\section{Proximal  operators of various non-smooth co-regularizers}
\label{sec: Section 3}
For a wider applicability of the proximal-based approach~\eqref{eq: algor}, and to facilitate comparisons among different sparsity-based co-regularizers, we provide in this section closed-form expressions for various forms of weighted sum of non-smooth functions. Specifically, we address: $i)$ elastic net functions, $ii)$ $\ell_0$-norm functions, and $iii)$ $\ell_1$-norm (which can be seen as a particular case of the elastic net) functions.
\vspace{-0.2 cm}
\subsection{Proximal operator of weighted sum of elastic net regularizations}
\label{subsec: elastic net norm}
The elastic net regularization combines $\ell_1$-norm and squared $\ell_2$-norm, promoting grouping effects~\cite{Zou2005regularization}. Specifically, it encourages both sparsity (via the $\ell_1$-norm) and smoothness (via the squared $\ell_2$-norm), and is defined as:
\vspace{-0.1 cm}
\begin{equation}
\vspace{-0.2 cm}
f(x)=\lVert x \rVert_1 + \frac{ \beta }{2} \lVert x \rVert_2^2,
\end{equation}
where $\beta>0$ is a regularization parameter controlling the relative contribution of the squared $\ell_2$-norm. As co-regularizer, the function $\boldsymbol{g}_{k,i}(w_k)$ in~\eqref{eq: gpsi} takes the following form:
 \begin{equation}
 \vspace{-0.2 cm}
 \label{eq: elastic norm with beta}
 \boldsymbol{g}_{k,i}(w_k)=\sum_{ \ell \in \cN_k} p_{k\ell} \left(\lVert w_k - \bpsi_{\ell,i} \rVert_1 + \frac{ \beta }{2} \lVert w_k - \bpsi_{\ell,i}  \rVert_2^2 \right),
\end{equation}
which can be re-written alternatively as:
\vspace{-0.3 cm}
\begin{align}
\vspace{-0.1 cm}
\boldsymbol{g}_{k,i}(w_k)&=\sum_{ \ell \in \cN_k} p_{k\ell} \Big( \sum_{m=1}^{M} \left \vert [w_k]_m - [\bpsi_{\ell,i}]_m \right \vert \notag\\
& \qquad + \frac{\beta}{2}\left([w_k]_m - [\bpsi_{\ell,i}]_m \right)^2 \Big) \notag\\
&= \sum_{m=1}^{M} \bphi_{k,i,m}\left([w_k]_m \right),\label{eq: elastic norm g}
\vspace{-0.6 cm}
\end{align}
\vspace{-0.4 cm}
where
\begin {align}
\vspace{-0.2 cm}
& \bphi_{k,i,m}([w_k]_m)= \sum_{ \ell \in \cN_k} p_{k\ell} \Big(\left \vert [w_k]_m - [\bpsi_{\ell,i}]_m \right \vert \notag\\
& \qquad + \frac{\beta}{2} \left([w_k]_m - [\bpsi_{\ell,i}]_m \right)^2 \Big).\label{eq: phi k i m}
\vspace{-0.3 cm}
\end{align}
Since $\bg_{k,i}(w_k)$ in~\eqref{eq: elastic norm g} is fully separable, its proximal operator can be evaluated component-wise  according to~\cite{Parikh2014proximal}: 
\vspace{-0.1 cm}
 \begin{align}
 \vspace{-0.2 cm}
 &\left[\prox_{\mu\eta \boldsymbol{g}_{k,i}} (\bpsi_k(i))\right]_m \notag\\
 & \quad = \prox_{\mu\eta \bphi_{k,i,m}} \left([\bpsi_k(i)]_m\right), \qquad \forall m=1,...,M,\label{eq: proximal operator of fully separable function}
 \vspace{-0.4 cm}
 \end{align}
 where $\bphi_{k,i,m}:\mathbb{R}\rightarrow \mathbb{R}$ is given by~\eqref{eq: phi k i m}.
 
 For simplicity, let us derive the proximal operator of the function $h(\cdot)$, which has a form similar to $\bphi_{k,m,i}(\cdot)$:
 \vspace{-0.2 cm}
 \begin{equation}
 \vspace{-0.2 cm}
 \label{eq: h}
  h(x)= \sum_{j=1}^{J}c_j \left(\vert x - b_j \vert + \frac{\beta}{2} (x - b_j)^2 \right),
  \end{equation}
  where $c_j > 0$ for all $j$,  and $b_1 < ...< b_J$. This ordering is assumed for convenience of derivation and does not affect the final result. From the optimality condition of~\eqref{eq: prox}, we have: 
\begin{equation}
0 \in \partial h\left(\prox_{\gamma h}\left(v\right)\right) + \frac{1}{\gamma}\left(\prox_{\gamma h}\left(v\right) - v\right),
\end{equation}
\vspace{-0.2 cm}
or, alternatively,
\vspace{-0.2 cm}
\begin{equation}
\vspace{-0.1 cm}
\label{eq: opt}
v - \prox_{\gamma h}(v) \in \gamma \partial h\left(\prox_{\gamma h}(v)\right).
\end{equation}
Since $x \in \mathbb{R}$, and $c_j$ and $\beta$ are $\geq 0$, we have \cite[Lemma 10]{Polyak2021basics}: 
\begin{align}
&\partial h(x)= \partial \left(\sum_{j=1}^{J}c_j \left(\vert x - b_j \vert + \frac{\beta}{2} \left(x - b_j \right)^2\right)\right) \notag\\
& \quad = \sum_{j=1}^{J}c_j \left( \partial \vert x - b_j \vert + \frac{\beta}{2} \nabla \left(x - b_j \right)^2 \right).
\end{align}
\vspace{-0.1 cm}
By using the fact that:
\vspace{-0.2 cm}
\begin{equation}
\partial |x-b_j|=\left\lbrace\begin{array}{ll}
1,&\text{if }x>b_j\\
\left[-1,1\right],&\text{if }x=b_j\\
-1,&\text{if }x<b_j
\end{array}
\right.
\end{equation}
we can show that the subdifferential of the real valued convex function $h(x)$ defined in~\eqref{eq: h} is given by: 
\vspace{-0.2 cm}
\begin{align}
\vspace{-0.2 cm}
\label{eq: subh}
&\partial h(x)= \notag \\
&\begin{cases}
  -\sum\limits_{j=1}^{J}c_j + \beta \sum\limits_{j=1}^{J} c_j\left(x -b_j \right), &\text{if } x < b_1 \\
  c_1[-1,1] - \sum\limits_{j=2}^{J}c_j + \beta \sum\limits_{j=2}^{J}c_j\left(b_1 - b_j \right), & \text{if }  x=b_1 \\
   c_1 -\sum\limits_{j=2}^{J}c_j + \beta \sum\limits_{j=1}^{J}c_j\left(x -b_j \right), &\text{if } b_1<x<b_2 \\
 \vdots\\
 \sum\limits_{j=1}^{J-1}c_j + c_J [-1,1]  + \beta\sum\limits_{j=1}^{J-1} c_j\left(b_J - b_j \right) , & \text{if } x=b_J  \\
\sum\limits_{j=1}^{J} c_j + \beta \sum\limits_{j=1}^{J}c_j\left(x -b_j\right), &\text{if } x>b_J .
\end{cases} 
\vspace{-0.5 cm}
\end{align}
From \eqref{eq: opt} and \eqref{eq: subh}, extensive but routine calculations lead to the following implementation for evaluating the proximal operator of $h$ in \eqref{eq: h}. Let us decompose $\mathbb{R}$ into $J + 1$ intervals such that $\mathbb{R}=\bigcup_{n=0}^{J} \mathcal{I}_n$ where: 
\vspace{-0.2 cm}
\begin{align}
\vspace{-0.4 cm}
\mathcal{I}_0 \quad& \triangleq\quad  \Big] -\infty , b_1 - \gamma \sum_{j=1}^{J} c_j + \beta \gamma \sum_{j=2}^{J}c_j  \left(b_1 - b_j\right) \Big[, \\
\mathcal{I}_n \quad& \triangleq \quad \mathcal{I}_{n,1} \cup \mathcal{I}_{n,2} \qquad  n = 1, \ldots, J ,
\vspace{-0.6 cm}
\end{align}
with 
\vspace{-0.2 cm}
\begin{align}
\vspace{-0.4 cm}
\mathcal{I}_{n,1} &\triangleq \Bigg[ b_n - \gamma\left(\sum_{j=n}^{J} c_j - \sum_{j=1}^{n-1}c_j\right) + \beta\gamma \sum_{j=1}^{J} c_j(b_n  - b_j),  \notag\\
&   b_n - \gamma \left(\sum_{j=n+1}^{J} c_j - \sum_{j=1}^{n}c_j\right) + \beta \gamma \sum_{j =1}^{J} c_j(b_n - b_j) \Bigg[, \notag\\
&\qquad  \qquad\qquad \qquad  \qquad\qquad\qquad  \qquad\qquad  n = 1, \ldots, J, \\
\mathcal{I}_{n,2} &\triangleq \Bigg[ b_n - \gamma\left(\sum_{j=n+1}^{J} c_j - \sum_{j=1}^{n}c_j\right) + \beta\gamma \sum_{j =1}^{J} c_j(b_n  -b_j),  \notag\\
& b_{n+1} - \gamma \left(\sum_{j=n+1}^{J} c_j - \sum_{j=1}^{n}c_j\right) + \beta \gamma \sum_{j =1}^{J} c_j(b_{n+1} - b_j)\Bigg[, \notag\\
&\qquad  \qquad\qquad \qquad  \qquad\qquad\qquad  \qquad n = 1, \ldots, J-1,\\
\mathcal{I}_{J,2} &\triangleq \Bigg[ b_J + \gamma\sum_{j=1}^{J} c_j + \beta\gamma \sum_{j =1}^{J}c_j \left(b_J  -  b_j\right), +\infty \Bigg[.
\end{align}
\vspace{-0.1 cm}
Then, depending on the interval to which $v$ belongs, we can evaluate the proximal operator according to:
\vspace{-0.2 cm}
\begin{align}
\label{eq: proxelastic}
&\prox_{\gamma h}(v)= \notag \\
&\quad \left\lbrace\begin{array}{ll}
  \frac{v + \gamma\left( \sum\limits_{j=1}^{J}c_j \right)+ \beta\gamma\left(   \sum\limits_{j=1}^{J}c_jb_j\right)}{ 1 + \beta\gamma\left(\sum\limits_{j=1}^{J} c_j \right)},  & \text{if } v \in \mathcal{I}_0 \\
  b_n ,& \text{if } v \in \mathcal{I}_{n,1} \\
 \frac{v + \gamma \left(\sum\limits_{j=n+1}^{J}c_j - \sum\limits_{j=1}^{n}c_j \right) + \beta \gamma\left( \sum\limits_{j=1}^{J}c_jb_j\right)}{ 1 + \beta\gamma\left(\sum\limits_{j=1}^{J}c_j \right)},   & \text{if } v \in \mathcal{I}_{n,2}.
 \end{array}\right.
\end{align}
\subsection{Proximal operator of weighted sum of $\ell_0$-norms}
\label{subsec: l0 norm}
Consider neighboring nodes $k$ and $\ell$, and let $\delta_{k,\ell}$ denote the difference vector $w_k - w_{\ell}$. When the parameter vectors across agents share many similar entries and only a relatively small number of distinct entries, the sparsity of $\delta_{k,\ell}$ can be promoted by considering its pseudo $\ell_0$-norm, which counts the number of nonzero entries. Although $\|\delta_{k,\ell}\|_0$ is a non-convex co-regularizer, presents computational challenges, and falls outside the scope of the analysis conducted in the previous section, we derive in this section a closed-form expression for the proximal operator of a weighted sum of $\ell_0$ -norms. This derivation is motivated by its optimality within the class of sparsity-based co-regularizers.
 
  When employed as co-regularizer, the function $\boldsymbol{g}_{k,i}(w_k)$ in~\eqref{eq: gpsi} takes the following form~\cite{Beck2017first-order}:
  \vspace{-0.1 cm}
 \begin{align}
 \vspace{-0.1 cm}
 \boldsymbol{g}_{k,i}(w_k)&=\sum_{ \ell \in \cN_k} p_{k\ell}\left(\lambda\lVert w_k - \bpsi_{\ell,i} \rVert_0\right) \notag\\
& = \sum_{\ell \in \mathcal{N}_k} p_{k\ell} \left(\sum_{m=1}^{M} I\left([w_k - \bpsi_{\ell,i}]_m\right)\right)\notag\\ 
& = \sum_{m=1}^{M} \bphi_{k,i,m}\left([w_k]_m\right),
\vspace{-0.3 cm}
\end{align}
where $\lambda$ is some positive parameter, the function $I:\mathbb{R}\rightarrow\mathbb{R}$ is defined as:
\vspace{-0.1 cm}
\begin{equation}
\label{eq: showing lambda}
I(x)=\left\lbrace\begin{array}{ll}
\lambda,&\text{if }x\neq 0\\
0,&\text{if }x= 0
\end{array}\right.
\end{equation} 
and the function $\bphi_{k,i,m}:\mathbb{R}\rightarrow\mathbb{R}$ is defined as:
\begin{equation}
\vspace{-0.2 cm}
\bphi_{k,i,m}\left([w_k]_m\right)= \sum_{\ell \in \cN_k} p_{k\ell}  I\left([w_k - \bpsi_{\ell,i}]_m\right). \label{eq: phi in the case of the l0 norm}
\end{equation} 
Similarly to the previous example, since $\boldsymbol{g}_{k,i}(w_k)$ is fully separable, its proximal operator can be evaluated component-wise according to~\eqref{eq: proximal operator of fully separable function} with $\bphi_{k,i,m}(\cdot)$ given by~\eqref{eq: phi in the case of the l0 norm}. By using similar arguments as those used in~\cite[Example 6.10]{Beck2017first-order}, we can show that the proximal operator of the function $\bphi_{k,i,m}(\cdot)$ is equal to the proximal operator of the function $\bphi'_{k,i,m}(\cdot)$:
\begin{equation}
\vspace{-0.2 cm}
\bphi'_{k,i,m}\left([w_k]_m\right)= \sum_{\ell \in \cN_k} p_{k\ell}  J\left([w_k - \bpsi_{\ell,i}]_m\right), \label{eq: phi' in the case of the l0 norm}
\end{equation} 
which is defined in terms of the function $J(x)\triangleq I(x)-\lambda$.

For simplicity, we derive  the proximal operator of the function $h_0(\cdot)$, which has a form similar to $\bphi'_{k,i,m}(\cdot)$: 
\vspace{-0.3 cm}
 \begin{equation}
 \vspace{-0.2 cm}
 \label{eq: definition of h_0 x}
h_0(x)= \sum_{j=1}^{J}c_j J(x-b_j),
 \end{equation}
where $\{c_j\}$ are non-negative coefficients, and $b_1\neq\ldots \neq b_J$ are real-valued parameters.  From~\eqref{eq: prox}, note that: 
\vspace{-0.1 cm}
\begin{equation}
\vspace{-0.2 cm}
 \prox_{\gamma h_0}(v) =  \arg\min_{x \in \mathbb{R}} \widetilde{h}_0(x,v),
\end{equation}
\vspace{-0.2 cm}
  where
\begin{align}
   \widetilde{h}_0(x,v) &\triangleq h_0(x) +  \frac{1}{2 \gamma} (x - v)^2 \notag \\
 &   \overset{ \eqref{eq: definition of h_0 x}}= 
\begin{cases} 
  \frac{1}{2\gamma}(x-v)^2 ,& \text{if } x \neq b_1  \neq \dots \neq b_J\\
  -\lambda c_1 + \frac{1}{2\gamma}(b_1-v)^2, & \text{if } x=b_1\\
  -\lambda c_2 + \frac{1}{2\gamma}(b_2-v)^2 ,& \text{if } x=b_2\\
  \vdots\\
  -\lambda c_J + \frac{1}{2\gamma}(b_J-v)^2, & \text{if } x=b_J.
\end{cases}
\end{align}
Evaluating the proximal operator requires distinguishing between two cases. The \emph{first case} arises when there exists an index $i\in\{1,\ldots, J\}$ such that $v=b_i$. In this case, we have:
\vspace{-0.2 cm}
\begin{align}
&\prox_{\gamma h_0}(v) = \notag \\
& \arg\min_{x \in \mathbb{R}} 
\left\lbrace \begin{array}{ll}
    \frac{1}{2\gamma}(x-b_i)^2, \qquad \qquad \text{if } x \neq b_1 \neq \dots \neq b_J \\
  -\lambda c_j + \frac{1}{2\gamma}(b_j-b_i)^2, \text{ if }  x=b_j, \quad j=1,\ldots,J
  \end{array}
\right.
\vspace{-0.2 cm}
\end{align}
Let $f(b_j)\triangleq  -\lambda c_j + \frac{1}{2\gamma}(b_j-b_i)^2$ and let $f_{\min}\triangleq \min_{b_1,\ldots,b_J}f(b_j)$. If we let $\Omega$ denote the set of $\{b_j\}$ for which  $f(b_j)=f_{\min}$, then we have:
\vspace{-0.2 cm}
\begin{equation}
\vspace{-0.1 cm}
\prox_{\gamma h_0}(v)=\{ b_j\in\Omega\}.
\end{equation}
%
The \emph{second case} arises when $ v \neq b_j \text{ , }\forall j=1,\ldots,J $. In this case, we have:
\vspace{-0.2 cm}
\begin{align}
\vspace{-0.2 cm}
&\prox_{\gamma h_0}(v)= \notag \\
%
 & \quad \arg\min_{x \in \mathbb{R}}   \begin{cases}
   \frac{1}{2\gamma}(x-v)^2  ,& \text{if } x \neq b_1 \neq \dots \neq b_J, \\
   -\lambda c_j + \frac{1}{2\gamma}(b_j-v)^2 ,& \text{if } x=b_j,\quad j=1,\ldots,J .
\end{cases}
\vspace{-0.3 cm}
\end{align}
Similarly, let $f(b_j)\triangleq  -\lambda c_j + \frac{1}{2\gamma}(b_j-v)^2$ and let $f_{\min}\triangleq \min_{b_1,\ldots,b_J}f(b_j)$. If we let $\Omega$ denote the set of $\{b_j\}$ for which  $f(b_j)=f_{\min}$, then we have:
\begin{align}
&\prox_{\gamma h_0}(v)=    \notag \\
%
& \quad  \arg\min_{x \in \mathbb{R}}   \begin{cases}
   \frac{1}{2\gamma}(x-v)^2  ,& \text{if } x \neq b_1 \neq \dots \neq b_J, \\
   -\lambda c_j + \frac{1}{2\gamma}(b_j-v)^2 ,& \text{if } x=b_j\in\Omega.
\end{cases}
\vspace{-0.2 cm}
\end{align}
The minimum of $  \frac{1}{2\gamma}(x-v)^2$ is attained at $x=v$ with a minimal value of $0$. If $f_{\min}<0$ (i.e., $|v-b_j|<\sqrt{2\gamma\lambda c_j}$), then the minimizer is $b_j\in\Omega$. On the other hand, if $f_{\min}>0$ (i.e., $|v-b_j|>\sqrt{2\gamma\lambda c_j}$), then the unique minimizer is $v$. Finally, if $f_{\min}=0$, then $v$ and $b_j\in\Omega$ (i.e., $|v-b_j|=\sqrt{2\gamma\lambda c_j}$) are minimums. Consequently, we can write:
\vspace{-0.2 cm}
\begin{align}
&\prox_{\gamma h_0}(v) = \notag\\
&\begin{cases}
  v , \\
\quad    \text{if } v > \sqrt{2\gamma\lambda c_j} + b_j \quad \text{and} \quad v < -\sqrt{2\gamma\lambda c_j} + b_j,~b_j\in\Omega \\
  \{ v, b_j\in\Omega\} , \\
 \quad   \text{if } v = \sqrt{2\gamma\lambda c_j} + b_j \quad \text{~or~} \quad v = -\sqrt{2\gamma\lambda c_j} + b_j \\
  \{b_j\in\Omega \} , \\
\quad   \text{if } -\sqrt{2\gamma\lambda c_j} + b_j < v < \sqrt{2\gamma\lambda c_j} + b_j .\\
\end{cases}
\vspace{-0.4 cm}
\end{align}
\vspace{-0.6 cm}
\subsection{Proximal operator of weighted sum of $\ell_1$-norms and reweighted $\ell_1$-norms}
\label{subsec: reweighted l1-norm}
A widely used \emph{convex} alternative to the $\ell_0$-norm for promoting sparsity is the $\ell_1$-norm~\cite{Nassif2016proximal, Candes2008enhancing} which, when applied to the vector $\delta_{k,\ell}=w_k-w_{\ell}$, can be expressed as:
\vspace{-0.2 cm}
\begin{equation}
\vspace{-0.1 cm}
\label{l1}
\Vert \delta_{k,\ell} \Vert_1= \sum_{m=1}^{M} \vert [\delta_{k,\ell} ]_m \vert.
\end{equation}
Since the $\ell_1$-norm uniformly shrinks all the components of a vector and does not distinguish between zero and non-zero entries, it is common to employ the reweighted  $\ell_1$-norm, which is an adaptive weighted formulation of the $\ell_1$-norm~\cite{Candes2008enhancing}. It was designed to reduce the bias induced by the $\ell_1$-norm, and consequently, enhance the penalization of the non-zero entries of a vector. Given the weight vector ${\alpha}_{k,\ell} = [\alpha_{k\ell}^1, ... , \alpha_{k\ell}^M]$, with $\alpha_{k\ell}^m > 0$ for all $m$, the reweighted $\ell_1$-norm is defined as~\cite{Candes2008enhancing}: 
\vspace{-0.2 cm}
\begin{equation}
\vspace{-0.2 cm}
\label{rl1}
\Vert \delta_{k,\ell} \Vert_{1,\text{rew}} = \sum_{m=1}^{M} \alpha_{k\ell}^m \vert [\delta_{k,\ell} ]_m \vert.
\end{equation}
The weights $\alpha_{k\ell}^m$ are usually chosen according to:
\vspace{-0.1 cm}
\begin{equation}
\vspace{-0.1 cm}
\alpha_{k\ell}^m = \frac{1}{ \epsilon + \vert [\delta_{k,\ell}^o ]_m \vert } , \quad m=1, \ldots , M,
\end{equation}
where $ \delta_{k,\ell}^o = w_k^o - w_{\ell}^o$~\cite{Candes2008enhancing, Nassif2016proximal}. Since the optimum parameter vectors $\{w^o_k,w^o_{\ell}\}$ in~\eqref{eq: local minimizers} are not known beforehand, it is common to employ the following adaptive implementation for the {weights~\cite{Candes2008enhancing, Nassif2016proximal}:}
\vspace{-0.1 cm}
\begin{equation}
\vspace{-0.1 cm}
\boldsymbol{\alpha}_{k\ell}^m (i)= \frac{1}{ \epsilon + \vert [\bpsi_{k,i}-\bpsi_{\ell,i} ]_m \vert }, \qquad m=1,\ldots , M,
\end{equation}
at each iteration $i$, where $\epsilon$ is a small constant preventing the denominator from vanishing, and $\bpsi_{k,i}-\bpsi_{\ell,i}$ is the available estimate of $\delta_{k,\ell}^o$ at nodes $k$ and $\ell$, iteration $i$. 
When we employ the $\ell_1$ or reweighted $\ell_1$-norm as co-regularizer, the function $\boldsymbol{g}_{k,i}(w_k)$ in~\eqref{eq: gpsi} takes the following form:
\vspace{-0.2 cm}
\begin{align}
\vspace{-0.2 cm}
\boldsymbol{g}_{k,i}(w_k)&=\sum_{ \ell \in \cN_k} p_{k\ell} \left( \sum_{m=1}^{M} \boldsymbol{\alpha}_{k\ell}^m (i)\left \vert [w_k]_m - [\bpsi_{\ell,i}]_m \right \vert \right)\notag\\
&= \sum_{m=1}^{M} \bphi_{k,i,m}\left([w_k]_m \right),\label{eq: l1 norm g}
\end{align}
\vspace{-0.2 cm}
where
\vspace{-0.1 cm}
\begin {equation}
\bphi_{k,i,m}([w_k]_m)= \sum_{ \ell \in \cN_k}\boldsymbol{p}_{k\ell}^m (i)  \left \vert [w_k]_m - [\bpsi_{\ell,i}]_m \right \vert ,\label{eq: phi k i m in the l1 norm}
\vspace{-0.1 cm}
\end{equation}
with $\boldsymbol{p}_{k\ell}^m (i)= p_{k\ell}$ in the $\ell_1$-norm case and  $\boldsymbol{p}_{k\ell}^m (i) = p_{k\ell}\boldsymbol{\alpha}_{k\ell}^m (i) $ in the reweighted  case. By comparing~\eqref{eq: l1 norm g} and~\eqref{eq: phi k i m in the l1 norm} with~\eqref{eq: elastic norm g} and~\eqref{eq: phi k i m}, we can see that the proximal operator for the $\ell_1$-norm can be obtained from the proximal operator of the elastic-net regularization by setting $\beta$ to $0$. Similarly, by adjusting the weights and setting $\beta=0$, the proximal operator of the reweighted $\ell_1$-norm can also be computed.

For illustration purposes, we represent in Fig.~\ref{fig: proximal operator} the proximal operator of three real-valued functions. The first function consists of a weighted sum of  elastic net regularizations, and its proximal operator is computed according to~\eqref{eq: proxelastic}, and is illustrated in Fig.~\ref{fig: proximal operator} (left). The second one consists of a weighted sum of  $\ell_0$-norms, and its proximal operator is computed according to Sec.~\ref{subsec: l0 norm}, and is illustrated in Fig.~\ref{fig: proximal operator} (middle). The third function  consists of a weighted sum of  $\ell_1$-norms, and its proximal operator is computed according to~\eqref{eq: proxelastic} with $\beta=0$, and is illustrated in Fig.~\ref{fig: proximal operator} (right). 
\begin{figure*}[t]
\centering
\includegraphics[scale=0.33]{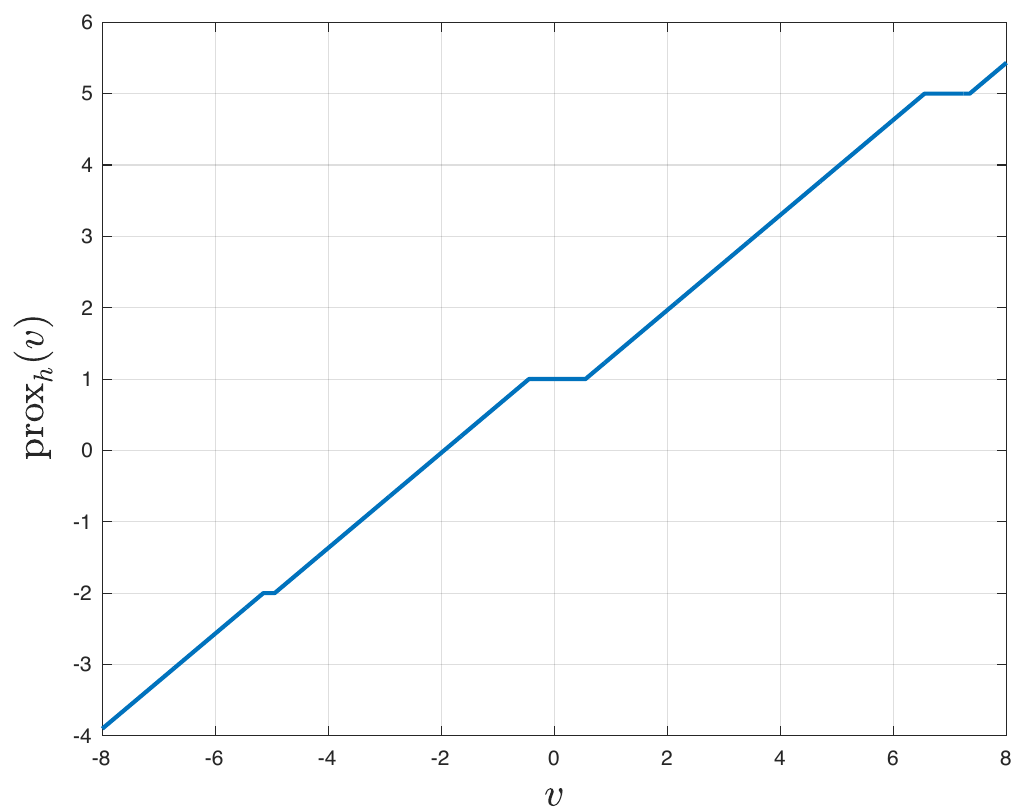}
\includegraphics[scale=0.33]{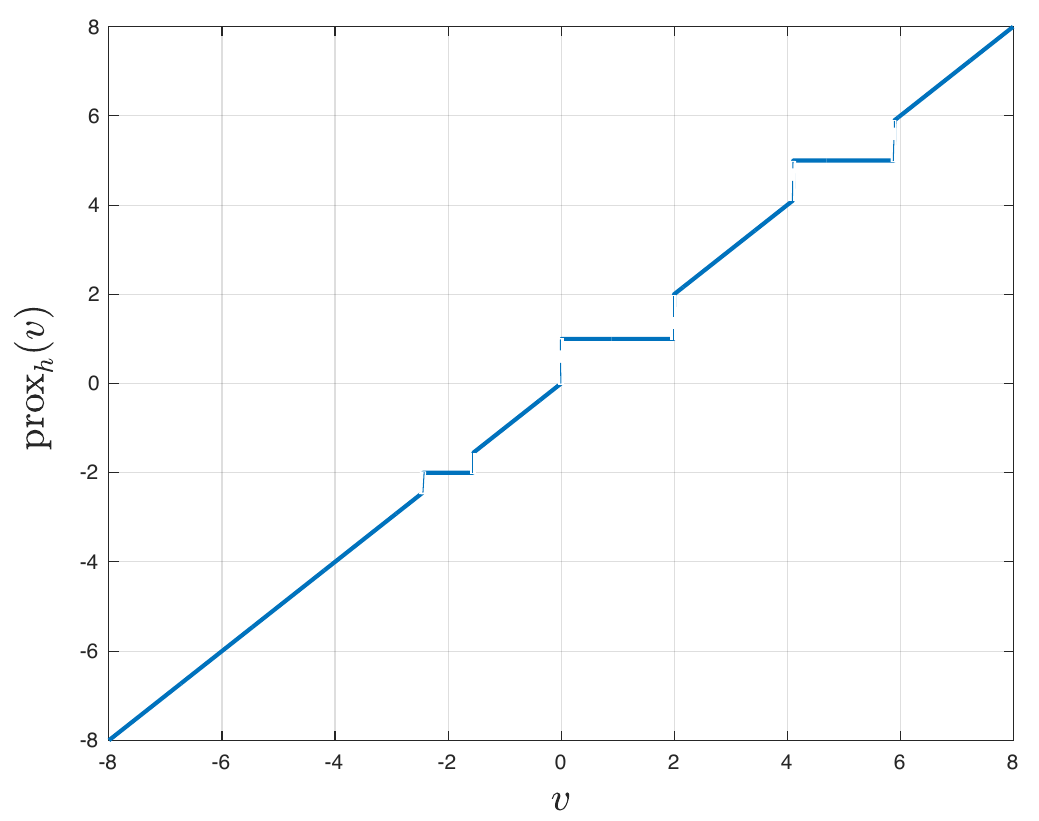}
\includegraphics[scale=0.33]{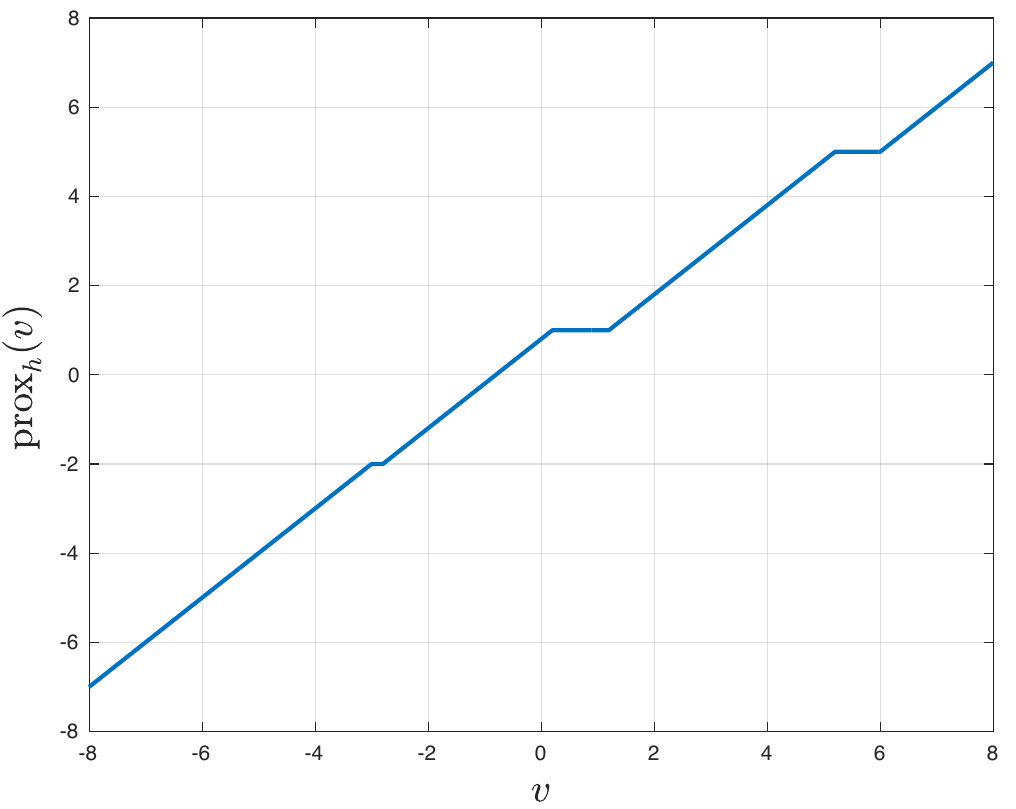}
 \caption{Proximal operator $ \prox_{\gamma h}(v)$ of different sparsity-based functions $h:\mathbb{R} \rightarrow \mathbb{R}$ (with $\gamma=1$). \emph{(Left)} Elastic net-based function where $h(x)= \frac{1}{10} \left(\left\vert x+2 \right\vert  +~\frac{1}{4}(x+~2)^2\right)+ \frac{1}{2} \left(\left\vert x-1 \right\vert  +~\frac{1}{4}(x-1)^2\right) + \frac{2}{5}\left(\left\vert x-5 \right\vert +  \frac{1}{4}(x-5)^2\right)$. \emph{(Middle)} $\ell_0$-based function where  $h(x)= \frac{1}{10} \left\Vert x+2 \right\Vert_{0} + \frac{1}{2} \left\Vert x-1 \right\Vert_{0} + \frac{2}{5}\left\Vert x-5 \right\Vert_{0}$. \emph{(Right)}  $\ell_1$-based function where  $h(x)= \frac{1}{10} \left\vert x+2 \right\vert + \frac{1}{2} \left\vert x-1 \right\vert + \frac{2}{5}\left\vert x-5 \right\vert$.}
  \label{fig: proximal operator}
\end{figure*}



\section{Simulations}
\label{sec: Section 4}
In this section, we illustrate the performance of algorithm~\eqref{eq: algor} in the context of mean-square-error (MSE) and logistic regression (LR) networks. We begin by revisiting the definitions of the MSE and LR networks~\cite{Sayed2014adaptation}.
\begin{example}{\emph{(Logistic regression network).}}
Consider agent $k$. Let $\boldsymbol{\gamma}_k(i)$ be a streaming sequence of (class) binary random variables that assume the values $\pm 1$, and let $\boldsymbol{h}_{k,i}$ be the corresponding streaming sequence of  $M \times 1$ real random (feature) vectors with $R_{h,k} = \expec \boldsymbol{h}_{k,i}\boldsymbol{h}_{k,i}^{\top} > 0$. The processes $\{ \boldsymbol{\gamma}_k(i), \boldsymbol{h}_{k,i}\}$ are assumed to be wide-sense stationary. In these problems, we would like to construct a classifier to predict the label $ \boldsymbol{\gamma}_k(i)$ based on the knowledge of the feature vector $\boldsymbol{h}_{k,i}$. To that end, agent $k$ can estimate the vector $w_k^o$ that minimizes the regularized logistic risk function~\cite{Sayed2014adaptation}:
 \begin{equation}
 \label{eq: logistic}
 J_k(w_k) = \expec \ln \left( 1+ e^{-\gamma_k(i)\boldsymbol{h}_{k,i}^{\top}w_k}\right) + \frac{\rho}{2} \Vert w_k \Vert^2,
 \vspace{-0.1 cm}
 \end{equation}
 where $\rho > 0$ is a regularization parameter. Once $w_k^o$ is found, $\widehat{\boldsymbol{\gamma}}_k(i) = \operatorname{sign}(\boldsymbol{h}_{k,i}^{\top} w_k^o)$ can then be used as a decision rule to classify new features. Note first that the cost in~\eqref{eq: logistic} is strongly convex due to the regularization term $\frac{\rho}{2} \Vert w_k \Vert^2$. Moreover, by using approximation \eqref{eq: approx}, we find that the multitask cooperative strategy \eqref{eq: algor} reduces to:
 \vspace{-0.1 cm}
\begin{equation}
\left\{
\begin{aligned} 
    & \bpsi_{k,i} = (1 - \mu \rho) \bw_{k,i-1}  + \mu \gamma_k(i) \boldsymbol{h}_{k,i} \left( \frac{1}{ 1 + e^{\gamma_k(i)\boldsymbol{h}_{k,i}^{\top}w_{k,i-1}}} \right) \\                          
    & \bw_{k,i}    =\operatorname{prox}_{\mu\eta \boldsymbol{g}_{k,i}} ( \bpsi_{k,i} )                                                                        
\end{aligned}
\right.
\end{equation}
\vspace{-0.1 cm}
where  $\boldsymbol{g}_{k,i}(w_k)$ is given by \eqref{eq: gpsi}.
\label{ex: lr}
\end{example}

\begin{example}{\emph{(Mean-square-error (MSE) network).}} In such networks, each agent $k$ is observing a streaming sequence   $\{\bd_k(i), \bu_{k,i}\}$ satisfying the linear regression model:
\begin{equation}
\vspace{-0.2 cm}
\label{eq: linear data model}
\bd_k(i) = \bu_{k,i}^{\top}w_k^o + \bv_k(i) ,
\end{equation}
where $w_k^o$ is some $M \times 1$ vector to be estimated by agent $k$ and $\bv_k(i)$ denotes a zero-mean measurement noise. 
The processes $\{\bu_{k,i}, \bv_k(i)\}$ are assumed to be zero-mean jointly wide-sense stationary with: i) $\expec \bu_{k,i}\bu_{\ell,i}^{\top} = R_{u,k} > 0$ if $k=\ell$ and zero otherwise; $\expec \bv_k(i)\bv_{\ell}(i) = \sigma_{v,k}^2$ if $k=\ell$ and zero otherwise; and iii) $\bu_{k,i}$ and $\bv_k(i)$ are independent of each other. 
To estimate $w^o_k$, each agent can minimize the MSE cost defined as~\cite{sayed2013diffusion}:
\vspace{-0.2 cm}
\begin{equation}
\label{eq: MSE}
J_k(w_k)= \frac{1}{2} \expec \left( \bd_k(i) - \bu_{k,i}^{\top}w_k \right)^2 ,
\end{equation}
which is minimized at $w_k^o$. Alternatively, agent $k$ can employ the cooperative strategy~\eqref{eq: algor} that takes the following form when approximation \eqref{eq: approx} is used:
 \begin{equation}
 \label{mseexap}
\left\{
\begin{aligned} 
  & \bpsi_{k,i} = \bw_{k,i-1} + \mu \bu_{k,i} (\bd_{k,i} - \bu_{k,i}^{\top} \bw_{k,i-1})\\                          
  & \bw_{k,i}    = \operatorname{prox}_{\mu\eta \boldsymbol{g}_{k,i}} ( \bpsi_{k,i} )                                                                           
\end{aligned}
\right.
\end{equation} 
where the function $\boldsymbol{g}_{k,i}(w_k)$ is given by~\eqref{eq: gpsi}.
\label{ex: mse}
\end{example}
\vspace{-0.3 cm}
\subsection{Illustrative simulations}
\label{subsec: illustrative simulations}
We start by considering an MSE network of $K=20$ agents, with the topology shown in Fig.~\ref{fig: experimental setup} \emph{(left)}. The regressors $\bu_{k,i}$ are $10 \times 1$ zero-mean Gaussian with covariance matrices $ \bR_{u,k}= \sigma_{u,k}^2 I_{10}$, where the variances $ \sigma_{u,k}^2$ are randomly generated from the uniform distribution $\mathcal{U}(1, 1.5)$.  The noise variables $\bv_k(i)$ are zero-mean Gaussian with variances $\sigma_{v,k}^2$ generated from the uniform distributions $\mathcal{U}(0.15, 0.25)$. The variances $\{\sigma_{u,k}^2 ,\sigma_{v,k}^2 \}$  are illustrated in Fig.~\ref{fig: experimental setup} \emph{(right)}. The regularization weights are set to $\rho_{k\ell} = \frac{1}{\text{card}\left\{\cN_k \right\}}$ for $\ell \in \cN_k$, where $\text{card}\left\{\cN_k \right\}$ represents the cardinality of the set $\cN_k$. The models $w^o_k$ in~\eqref{eq: linear data model} are generated according to: 
\vspace{-0.1 cm}
\begin{equation}
\vspace{-0.1 cm}
w^o_k=w_{\text{c}}+\delta_k,\qquad k=1,\ldots,20,\label{eq: local models generated sparsely}
\end{equation}
 where $w_{\text{c}}$ is a $10\times 1$ vector randomly generated from the Gaussian distribution $\cN(0,I_{10})$. The vectors $\{\delta_k\}_{k=1}^{10}$ are  $10\times 1$ with all entries equal to $0$ except for the $k$-th entry, which is equal to $1$. Similarly, the vectors $\{\delta_k\}_{k=11}^{20}$ are  $10\times 1$ with all entries equal to $0$ except for the $(k-10)$-th entry, which is equal to $-1$. Consequently, the individual models differ only sparsely across the nodes.

\begin{figure}
\begin{center}
\includegraphics[scale=0.1]{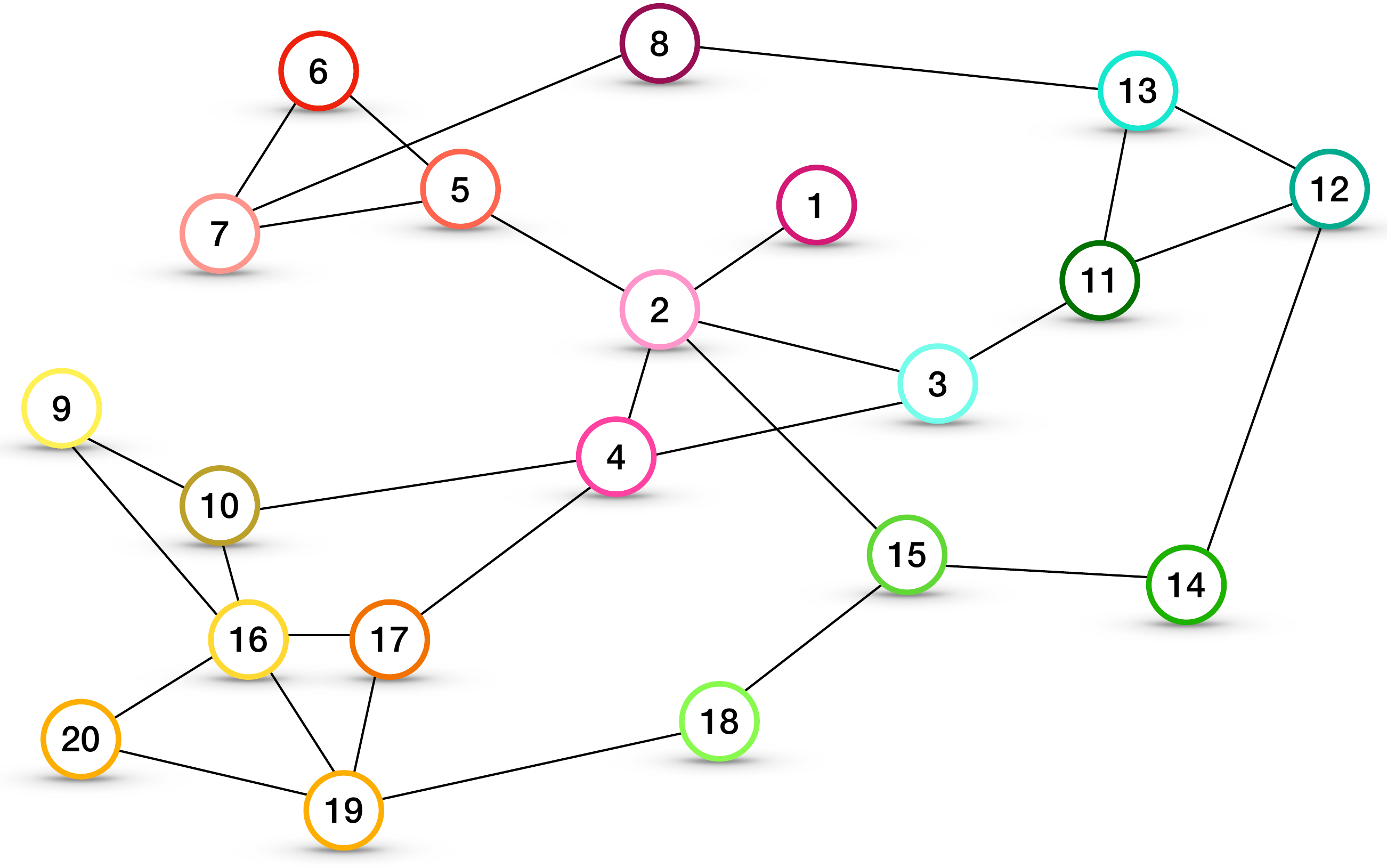}
\includegraphics[scale=0.2]{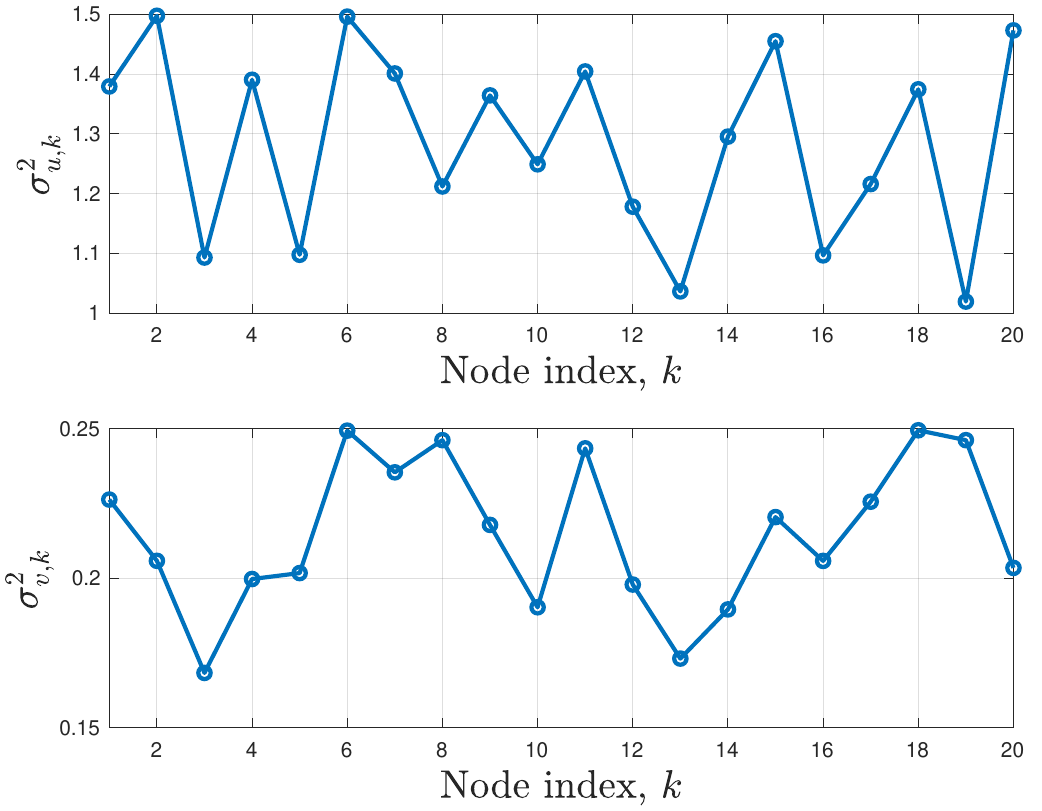}
 \caption{Experimental setup. \emph{(Left)} Network topology.  \emph{(Right)} Regression and noise variances. }
  \label{fig: experimental setup}
\end{center}
\end{figure}
\subsubsection{Illustration of Theorem~\ref{theorem2}} In order to illustrate the theoretical findings in Theorem~\ref{theorem2}, we consider the $\ell_1$-norm co-regularizer (i.e., $f_{k\ell}(w_k,w_{\ell})=\|w_k-w_{\ell}\|_1$) which satisfies Assumption~\ref{ass: ass3}. We set $\eta=\kappa\mu^{\alpha}$ with $\kappa=50$ and $\alpha=1$. In Fig.~\ref{fig: theorem illustration}, we report the network mean-square-deviation (MSD) learning curves:
\vspace{-0.3 cm}
\begin{equation}
\vspace{-0.1 cm}
\label{eq: msd w.r. weta}
\text{MSD}(i)=\frac{1}{K}\sum_{k=1}^K\expec\|w^o_{k,\eta}-\bw_{k,i}\|^2,
\end{equation}
of  algorithm~\eqref{eq: algor}, for $3$ different values of the step-size $\mu=\{\mu_0/2,\mu_0,2\mu_0\}$. The results are averaged over $50$ Monte-Carlo runs. For a given $\eta$, we compute the reference vector $\cw_{\eta}^o$ in~\eqref{eq: global}, which is needed for the performance evaluation in~\eqref{eq: msd w.r. weta}, using the CVX solver~\cite{grant2014cvx} in deterministic settings where the data distribution is known. As it can be observed from Fig.~\ref{fig: theorem illustration},  in steady-state, the network MSD increases by approximately $3$ dB when $\mu$ goes from $\mu_0$ to $2\mu_0$. This
means that the performance is on the order of $\mu$, as expected from Theorem~\ref{theorem2}.
\begin{figure}
    \centering
    \includegraphics [scale=0.4]{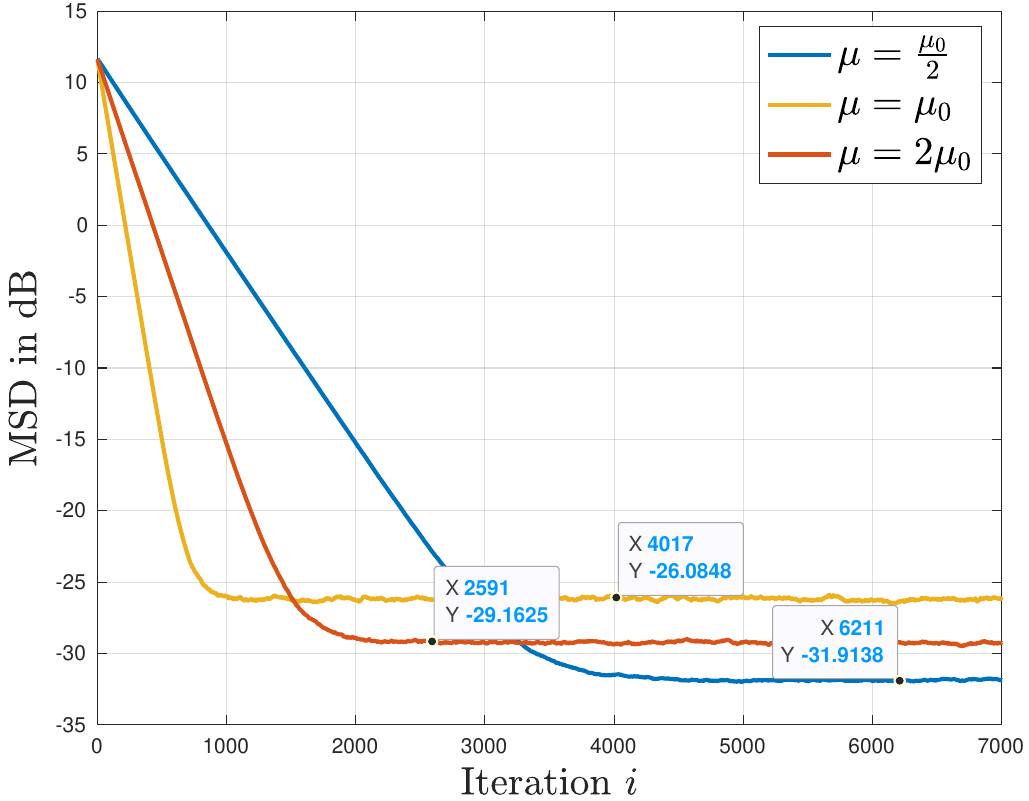}
    \caption{Performance of  algorithm~\eqref{eq: algor} w.r.t. $\cw^o_{\eta}$ in~\eqref{eq: global} when $\ell_1$-norm co-regularization is used. Evolution of the MSD learning curves~\eqref{eq: msd w.r. weta} for three different values of the step-size ($\mu_0=0.0025$).}
    \label{fig: theorem illustration}
\end{figure}
\begin{figure}
\begin{center}
\includegraphics[scale=0.24]{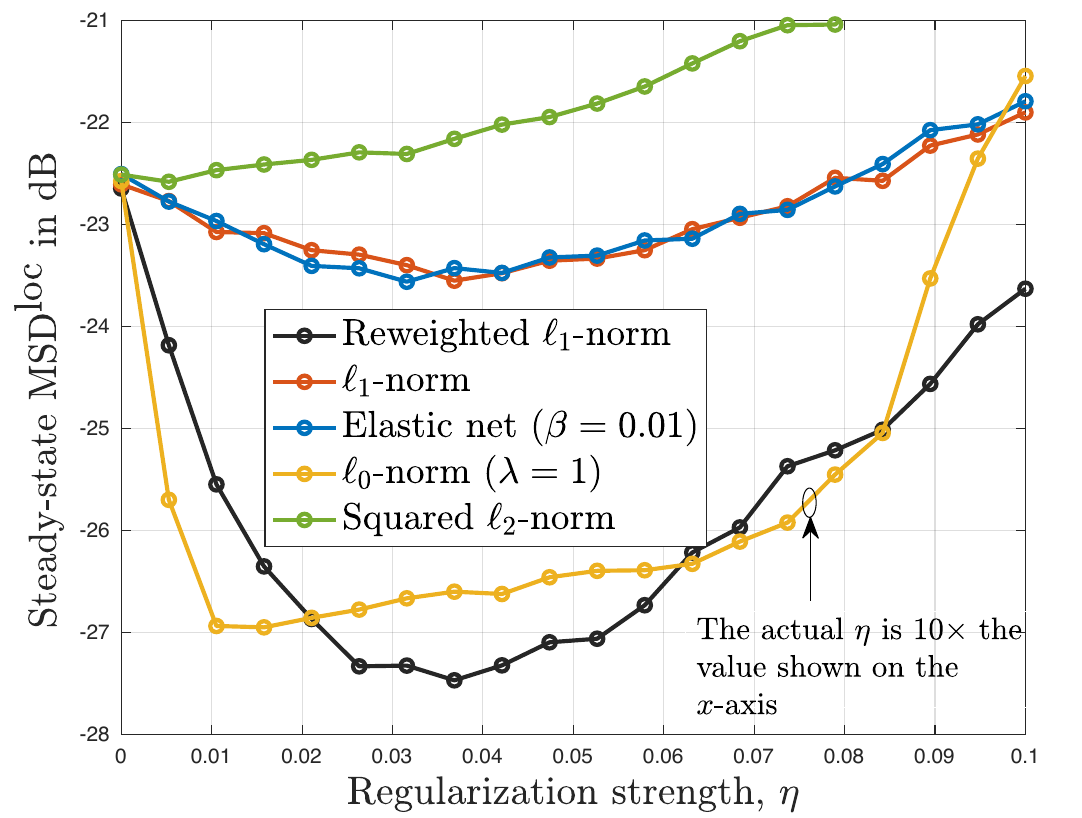}
\includegraphics [scale=0.24]{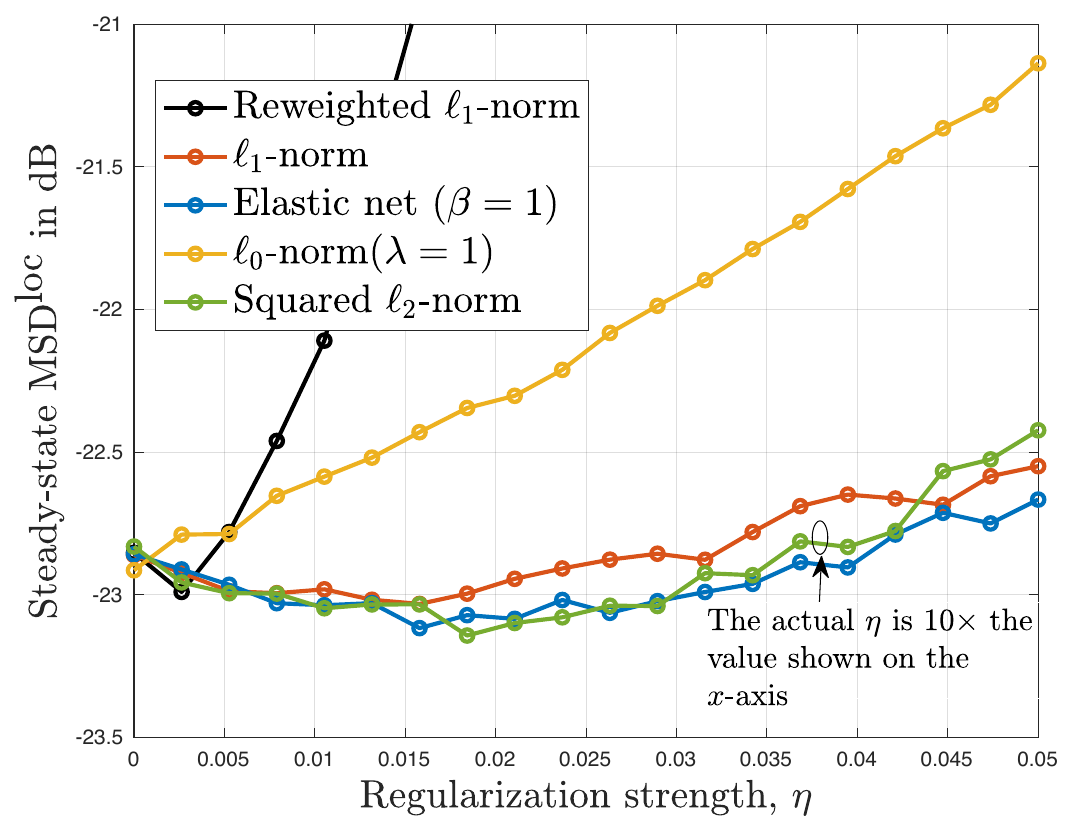}
  \caption{Benefit of collaborative learning evaluated in terms of the mean-square-error relative to the individual models $w^o_k$ in~\eqref{eq: local minimizers}. \emph{(Left)} Models $\{w^o_k\}$ differ sparsely across the nodes. \emph{(Right)} Models $\{w^o_k\}$ vary smoothly over the graph.}
 \label{fig:benefit of collaborative learning}
\end{center}
\end{figure}
\subsubsection{Benefit of collaborative learning}
To illustrate the benefit of collaborative learning, we assess the algorithm's performance w.r.t. the local minimizers $w^o_k$ in~\eqref{eq: local minimizers}, which correspond to the local models in~\eqref{eq: linear data model}. Specifically, rather than analyzing the MSD w.r.t. $\cw^o_{\eta}$ as in~\eqref{eq: msd w.r. weta}, we focus on the mean-square-error relative to $\{w^o_{k}\}$:
\vspace{-0.2 cm}
\begin{equation}
\vspace{-0.2 cm}
\label{eq: msd w.r. wek}
\text{MSD}^{\text{loc}}(i)=\frac{1}{K}\sum_{k=1}^K\expec\|w^o_{k}-\bw_{k,i}\|^2,
\end{equation} 
where the superscript ``loc'' is used to emphasize that the error is evaluated w.r.t. the local models. We fix the step-size $\mu=0.005$ and report the steady-state value of $\text{MSD}^{\text{loc}}(i)$ for different values of the regularization strength~$\eta$. We compare various forms of co-regularization. Specifically, we examine $\ell_1$-norm co-regularization, where $f_{k\ell}(w_k,w_{\ell})=\|w_k-w_{\ell}\|_1$ (see Sec.~\ref{subsec: reweighted l1-norm}); reweighted $\ell_1$-norm co-regularization ($\epsilon=0.1$), as described in Sec.~\ref{subsec: reweighted l1-norm};  $\ell_0$-norm co-regularization, where  $f_{k\ell}(w_k,w_{\ell})=\|w_k-w_{\ell}\|_0$ (see Sec.~\ref{subsec: l0 norm}); elastic net co-regularization, where $f_{k\ell}(w_k,w_{\ell})=\|w_k-w_{\ell}\|_1+\frac{\beta}{2}\|w_k-w_{\ell}\|^2$ (see Sec.~\ref{subsec: elastic net norm}); and squared $\ell_2$-norm co-regularization, where $f_{k\ell}(w_k,w_{\ell})=\|w_k-w_{\ell}\|_2^2$, which is studied in details in~\cite{Nassif2020learning} and is shown to be effective in settings where the individual models vary smoothly across the graph.

 In the first experiment (see Fig.~\ref{fig:benefit of collaborative learning} \emph{(left)}), the individual models $\{w^o_k\}$ are generated as explained in Sec.~\ref{subsec: illustrative simulations}, ensuring they differ sparsely across the nodes. The elastic net regularization parameter~$\beta$ in~\eqref{eq: elastic norm with beta} is set to $0.01$, and the $\ell_0$-norm parameter $\lambda$ in~\eqref{eq: showing lambda} is set to $1$. The global  regularization strength~$\eta$ is uniformly sampled from the interval $[0,0.1]$. For each value of $\eta$, the resulting reported MSD$^{\text{loc}}$ is obtained by averaging the instantaneous mean-square-deviation $\text{MSD}^{\text{loc}}(i)$ in~\eqref{eq: msd w.r. wek} over  $200$ samples after convergence of the algorithm (the expectation in~\eqref{eq: msd w.r. wek} is estimated using $30$ Monte-Carlo runs for all norms, except for the $\ell_0$-norm, where the expectation is evaluated over $400$ Monte-Carlo runs). As it can be observed from Fig.~\ref{fig:benefit of collaborative learning}  \emph{(left)}, and as expected, when the individual models differ sparsely across the nodes, sparsity-based co-regularizers (such as the reweighted $\ell_1$-norm and $\ell_0$-norm) tend to outperform smoothness-based norms.  Specifically,  a $5$dB improvement can be achieved through collaboration (using reweighted $\ell_1$-norm co-regularizers) compared to the non-cooperative solution (obtained by setting $\eta=0$).
 
 In the second experiment (see Fig.~\ref{fig:benefit of collaborative learning} \emph{(right})), we generate the individual models $\{w^o_k\}$ such that they vary smoothly over the graph by following the approach in~\cite[Sec.~III-C]{Nassif2018regularization}, with the smoothness parameter  set to $\tau=5$. The elastic net regularization parameter~$\beta$ in~\eqref{eq: elastic norm with beta}, and the $\ell_0$-norm parameter $\lambda$ in~\eqref{eq: showing lambda} are both set to $1$. The global  regularization strength~$\eta$ is uniformly sampled from the interval~$[0,0.05]$. Similarly to the previous experiment, for each value of $\eta$, the resulting reported MSD$^{\text{loc}}$ is obtained by averaging the instantaneous mean-square-deviation $\text{MSD}^{\text{loc}}(i)$ in~\eqref{eq: msd w.r. wek} over  $200$ samples after convergence of the algorithm (the expectation in~\eqref{eq: msd w.r. wek} is estimated using $50$ Monte-Carlo runs for all norms, except for the $\ell_0$-norm, where the expectation is evaluated over $400$ Monte-Carlo runs). As it can be observed from Fig.~\ref{fig:benefit of collaborative learning}  \emph{(right)}, and as expected, when the individual models vary smoothly across the graph, smoothness-based co-regularizers (such as elastic net regularization and squared $\ell_2$-norm) tend to outperform sparsity-based norms. In particular, the reweighted $\ell_1$-norm and $\ell_0$-norm coregularization are less likely to outperform the performance of the  non-cooperative solution.

\subsection{Simulation results with real dataset }
\begin{table*}[t]
\centering
\begin{tabular}{c|c c c c c c c c}
\hline\hline
& $\eta=0$ & $\eta=1$ & $\eta=4$ & $\eta=100$ &$\eta=1000$ & $\eta=2000$ & $\eta=5000$ & $\eta=10000$ \\ \hline
Prediction Error ($\ell_1$-norm) & $0.2801$ & $0.2387$ & $\mathbf{0.2239}$ & $0.2283$ & $0.2294$&$0.2302$ & $0.2308$ & $0.2308$   \\
Prediction Error (Elastic Net regularization)   & $0.2801$  & $0.2295$& $\mathbf{0.2233}$  & $0.2279$ & $0.2298$ &$0.2304$ & $0.2307$  & $0.2309$ \\
Prediction Error (Reweighted $\ell_1$-norm) & $0.2801$ & $0.2483$ & $0.2384$ & $0.2330$ & $0.2301$ &$0.2311$ & $\mathbf{0.2293}$ & $0.2307$  \\
Prediction Error (Squared $\ell_2$-norm) & $0.2801$ & $0.2425$ & $0.2317$ & $\mathbf{0.2251}$ & $0.2287$ & $0.2313$ & - & -  \\
\hline\hline
\end{tabular}
\caption{Prediction error \eqref{eq: prederror} for different values of regularization strength $\eta$ using various coregularization norms.}
\label{table}
\end{table*}

In this section, we simulate algorithm \eqref{eq: algor} 
on a real weather dataset corresponding to a collection of daily measurements (mean temperature, mean dew point, mean visibility, mean wind speed, maximum sustained wind speed, and rain or snow occurrence) taken from $2004$ to $2017$ at $K=139$ weather stations located around the continental United States \cite{Menne2018global} -- the same dataset was previously used in~\cite[Sec. IV]{Nassif2020learning} in the context of smooth graph-based regularization. We construct a representation weighted graph $\mathcal{G} = (\mathcal{N}, \mathcal{E}, A)$ based on the geographical distances between sensors, and the $4$-nearest neighbors rule as in~\cite{Sandryhaila2013discrete,Nassif2020learning}. 
Let $\boldh_{k,i} \in \mathbb{R}^M$ denote the feature vector at sensor $k$ and day $i$ composed of $M=5$ entries corresponding to the mean temperature, mean dew point, mean visibility, mean wind speed, and maximum sustained wind speed reported at day~$i$ at sensor~$k$. Let $\bgamma_k(i)$ denote a binary variable associated with the occurrence of rain (or snow) at node $k$ and day $i$, i.e, $\bgamma_k(i)=1$ if rain (or snow) occurred and $\bgamma_k(i)=-1$ otherwise. As in~\cite[Sec. IV]{Nassif2020learning}, the objective is to construct a classifier that allows us to predict whether it will rain (or snow) or not based on the knowledge of the feature vector $\boldh_{k,i}$. Each station could use logistic regression as explained in Example 1 to learn the individual classifier. However, as we will see in the following, collaborative learning has the potential to improve the the network's performance.

We split the dataset into a training set used to learn the decision rule $w_k^o$, and a test set from which $\widehat{\boldsymbol{\gamma}}_k(i)$ are generated for performance evaluation. The training dataset comprises data recorded at the stations in the interval $2004 - 2012$ (a total number of $D_a=3288$ days), and the test set contains data recorded in the interval $2012 - 2017$ (a total number of $D_t=1826$ days). We run strategy \eqref{eq: algor} over the training set $(i=1,\ldots,D_a)$  using the $\ell_1$-norm, the reweighted $\ell_1$-norm and the elastic net regularization for different values of $\eta$. We set $ \mu=5.10^{-4}$, $\rho=10^{-5}$, 
and $\beta=1$ for the elastic net regularization. We generate the first iterate $w_{k,0}$ from the Gaussian distribution $\mathcal{N}(0, I_{M})$. For each value of $\eta$, we report in Table~\ref{table} the prediction error over the test set defined as~\cite[Sec. IV]{Nassif2020learning}: 
\begin{equation}
\vspace{-0.1 cm}
\label{eq: prederror}
\frac{1}{K} \sum_{k=1}^{K} \frac{1}{D_t} \sum_{i=1}^{D_t=1826} \mathbb{I} [ \sign (\boldsymbol{h}_{k,i}^{\top} \widehat{w}_{k,\infty}) \neq \bgamma_k(i)],
\vspace{-0.1 cm}
\end{equation}
where $ \widehat{w}_{k,\infty}$ is the average of the last $200$ iterates generated by algorithm~\eqref{eq: algor} at agent $k$, and $\mathbb{I} [x]$ is the indicator function at $x$, namely, $\mathbb{I} [x]=1$ if $x$ is true and $0$ otherwise. Table \ref{table} shows that, through cooperation, the agents improve their performance. In the non-cooperative solution ($\eta=0$), where each agent minimizes its cost locally, the error is the largest. 
As $\eta$ increases, the network prediction error improves for all coregularizers up to a certain point, after which it starts to increase again. This is because, for very large values of $\eta$, the agents shift their focus towards estimating a common classifier, which can result in a decrease in accuracy. On the other hand, our results indicate that, for this application, the elastic net regularization outperforms the other norms. This is due to its ability to effectively balance smoothness and sparsity.
\section{Conclusion}
In this work, we considered multitask learning over networks, where each agent seeks to estimate its own parameter vector, and where the parameter vectors at neighboring agents are assumed to differ sparsely. We showed how sparsity-based coregularization can be employed to perform collaborative learning. To handle non-smooth regularization, we employed proximal-based approaches. We showed that, in the small step-size regime, and under some conditions on the individual cost functions, gradient vector approximations, regularization functions, and regularization strength, the network can approach the minimizer of the global regularized problem with arbitrarily good accuracy levels. Additionally, for a broader applicability and to achieve higher computational efficiency, we derived closed-form expressions for computing the proximal operator of weighted sum of different sparsity-based norms. Finally, simulation results illustrated the theoretical findings, the efficiency of the proposed approach, and the benefit of collaborative learning.

\begin{appendices}
\section{Proof that $\ell_1$-norm co-regularization satisfies Assumption~\ref{ass: ass3}}
\label{Appendix: A}
We start by writing  $r_k(w_{k},  \{w_{\ell}\}_{ \ell \in \cN_k})=  \sum_{ \ell \in \cN_k} p_{k\ell} \|w_k-w_{\ell}\|_1$ as: 
\vspace{-0.2 cm}
\begin{align}
\vspace{-0.2 cm}
r_k\left(w_{k}, \left\{w_{\ell}\right\}_{ \ell \in \cN_k}\right) &=\sum_{ \ell \in \cN_k} p_{k\ell} \left( \sum_{m=1}^{M} \left \vert [w_k]_m - [w_{\ell}]_m \right \vert  \right) \notag \\
&= \sum_{m=1}^{M} \phi_{k,m} \left( [w_k]_m, \{[w_{\ell}]_m\}_{ \ell \in \cN_k}  \right),
\vspace{-0.4 cm}
\end{align}
\vspace{-0.3 cm}
where
\begin {equation}
\vspace{-0.1 cm}
\phi_{k,m}\left ([w_k]_m, \{[w_{\ell}]_m\}_{ \ell \in \cN_k}  \right)= \sum_{ \ell \in \cN_k} p_{k\ell} \left \vert [w_k]_m - [w_{\ell}]_m \right \vert .
\end{equation}
The subgradient of $ r_k(w_k)$ can be computed component-wise:
\begin{equation}
\vspace{-0.1 cm}
\partial_{[w_k]_m  }r_k(w_k) =  \partial_{[w_k]_m  }  \phi_{k,m}\left([w_k]_m \right),  \qquad \forall m=1,...,M.
\end{equation}
For clarity of presentation, we shall derive the subdifferential of a function $h(\cdot)$ similar to $\bphi_{k,m}$ where:
 \begin{equation}
 \vspace{-0.1 cm}
 \label{eq: h(x)}
  h(x)=  \sum_{j=1}^{J}c_j h_j(x)=\sum_{j=1}^{J}c_j \vert x - b_j \vert,
  \vspace{-0.1 cm}
  \end{equation}
  with $c_j > 0$ for all $j$. Since $x \in \mathbb{R}$, $c_j$ are non-negative, we have \cite[Lemma 10]{Polyak2021basics}: 
\begin{equation}
\vspace{-0.1 cm}
\partial h(x)=  \partial \left( \sum_{j=1}^{J}c_j h_j(x)\right) =  \sum_{j=1}^{J}c_j  \partial h_j(x)= \sum_{j=1}^{J}c_j  \partial \vert x - b_j \vert,
\end{equation}
which, from eq.~$(35)$ in~\cite[Sec. II-D]{Nassif2016proximal}, can be shown to be bounded and, therefore, the set of subgradients of the function $ r_k(w_{k},  \{w_{\ell}\}_{ \ell \in \cN_k})$ with respect to $w_k$ is uniformly bounded.


\end{appendices}
\bibliographystyle{IEEEbib}
{\balance{
\bibliography{reference}}}

\end{document}